\documentclass[twoside]{article}

\usepackage[preprint]{aistats2027}
\usepackage[round]{natbib}

\usepackage[utf8]{inputenc} % allow utf-8 input
\usepackage[T1]{fontenc}    % use 8-bit T1 fonts
\usepackage{hyperref}       % hyperlinks
\usepackage{url}            % simple URL typesetting
\usepackage{booktabs}       % professional-quality tables
\usepackage{amsfonts}       % blackboard math symbols
\usepackage{nicefrac}       % compact symbols for 1/2, etc.
\usepackage{microtype}      % microtypography
\usepackage{xcolor}         % colors
\usepackage{amsmath, amssymb, amsthm}
\usepackage[ruled,vlined]{algorithm2e} 
\newtheorem{lemma}{Lemma}
\usepackage{multirow}
\usepackage{graphicx}
\usepackage{float}
\usepackage{pifont}
\newcommand{\cmark}{\ding{51}}%
\newcommand{\xmark}{\ding{55}}%
\usepackage{mathtools}
\newtheorem{remark}{Remark} 
\usepackage{subcaption}
\usepackage{caption}
\usepackage{dblfloatfix}
\usepackage{cuted}
\usepackage{placeins}
\FloatBarrier

\usepackage{titlesec}
\titlespacing*{\subsection}{0pt}{6pt plus 1pt minus 1pt}{2pt plus 1pt}
\titlespacing*{\subsubsection}{0pt}{4pt plus 1pt minus 1pt}{1pt plus 1pt}
\titlespacing*{\paragraph}{0.2pt}{2pt plus 1pt minus 1pt}{1em}

\definecolor{Wgreen}{HTML}{01a049}

\newcommand{\warn}[1]{}

\begin{document}



\twocolumn[
\aistatstitle{SpatialUQ: Post-Hoc Uncertainty Quantification from Spatial Consistency in Black-Box Vision Models}

\aistatsauthor{ Md Kawsher Mahbub$^1$ \quad Milon Biswas$^2$ \quad Mirza Niaz Morshed$^3$ \quad Wei Yu$^2$ }

\aistatsaddress{ 
  $^1$NPI University of Bangladesh, $^2$Towson University, $^3$University of Asia Pacific \\
  \texttt{kawsher@npiub.edu.bd, mbiswas1@students.towson.edu, 25214015@uap-bd.edu, wyu@towson.edu}
}
]

\begin{abstract}
Clinical vision models are often deployed as frozen black boxes with no access to internals, retraining, or ground truth at inference time. We introduce \textbf{SpatialUQ}, a post-hoc uncertainty method using only output probabilities. It measures the Jensen-Shannon divergence between the global prediction and the mean of five fixed spatial crops in six deterministic forward passes. The premise is simple, trustworthy predictions are spatially consistent. On NIH ChestX-ray14 (DenseNet-121, $N{=}25{,}596$), our Multicrop Uncertainty Score (MUS) reaches $0.784$ failure-detection AUC versus $0.664$ for MC-Dropout ($p{<}10^{-6}$) at one-fifth the compute, with native calibration ($\text{SCE}{=}0.049$ vs.\ $0.127$ for $\ell_1$), the best-calibrated among methods above 0.78 AUC. A supervised fusion of MUS with entropy, confidence, and $\ell_1$ reaches $0.832$, outperforming a five-member ensemble ($0.813$). MUS scales with model quality, reaching $0.899$ with BiomedCLIP ($\rho = 0.846$), while this relationship remains meaningful in-distribution ($\rho = 0.523$) but breaks down under severe distribution shift (VinBigData, $\rho = 0.027$). MUS is well-suited to diffuse findings but is less dependable for small focal lesions such as nodules. Code and experimental materials are publicly available at
\url{https://huggingface.co/datasets/kawsher11/SpatialUQ}.

\end{abstract}

\section{Introduction}\label{sec:intoduction}

Deep neural networks can fail silently on out-of-distribution inputs~\citep{amodei2016concrete,hendrycks2019benchmarking}, posing patient-safety risks in medical imaging where confident errors are consequential~\citep{leibig2017leveraging,kendall2017uncertainties}. Reliable deployment therefore requires inference-time failure detection without ground truth. However, frozen black-box classifiers exclude many existing UQ approaches. Bayesian methods such as MC-Dropout~\citep{gal2016dropout} require training-time modifications, ensembles~\citep{lakshminarayanan2017simple} and TTA~\citep{krizhevsky2012imagenet} add substantial compute, while post-hoc methods such as DDU~\citep{mukhoti2023deep}, Mahalanobis~\citep{lee2018simple}, energy~\citep{liu2020energy}, ODIN~\citep{liang2017enhancing}, and conformal prediction~\citep{angelopoulos2023conformal} require model internals, feature distributions, or calibration data. Recent consistency-based~\citep{khan2024consistency} and augmentation-sensitivity~\citep{shu2026uncovering} methods operate in related settings but rely on linguistic or photometric perturbations, leaving geometric evidence largely unexplored.

We propose \textbf{SpatialUQ} grounded on a single premise: a reliable model produces \emph{spatially consistent} predictions across an image and its sub-regions. Well-localized features create strong agreement between global and crop-level predictions. On the other hand,  distributional mismatch, low image quality, or atypical presentations make them diverge. We measure this divergence as the Jensen-Shannon Divergence (JSD) between the global prediction and the mean of five fixed spatial crops, in six deterministic forward passes. This geometric signal needs no internals, sampling, or reference data, and differs fundamentally from TTA~\citep{krizhevsky2012imagenet,wang2019aleatoric} (Lemma~\ref{lemma:bound}; Section~\ref{sec:photometric_perturbation}).  A supervised logistic fusion of MUS with entropy, confidence, and $\ell_1$ improves failure detection further, except under severe shift where MC-Dropout remains preferable.

\textbf{Our key contributions are summarized as follows:}
\begin{itemize}
\item \textbf{SpatialUQ}: A post-hoc uncertainty framework for frozen black-box classifiers using only output probabilities, with no internals, gradients, training data, or retraining. MUS is natively calibrated and extends to frozen foundation models, reaching $0.899$ AUC with BiomedCLIP ($\rho=0.846$). Spatial perturbations provide the key signal, with photometric JSD reducing AUC from $0.784$ to $0.755$.

\item \textbf{Formalization and self-diagnosing deployment gate}: A formal bound connects MUS to mean absolute spatial probability shift, while a small labeled calibration sample provides a deployment diagnostic. The correlation remains meaningful in-distribution ($\rho=0.523$) but collapses under severe shift (VinBigData, $\rho=0.027$).

\item \textbf{State-of-the-art failure detection at one-fifth the cost}: MUS achieves $0.784$ AUC on NIH ChestX-ray14, outperforming MC-Dropout ($0.664$, $p<10^{-6}$) with 6 vs.\ 30 forward passes. Fusion reaches $0.832$ AUC, exceeding a five-member ensemble ($0.813$) without retraining.

\item \textbf{Cross-domain transferability}: MUS achieves $0.708$ AUC on CheXpert, a $0.109$ gain over MC-Dropout; fusion reaches $0.917$--$0.937$ on ImageNet and $0.885$--$0.913$ on MS COCO.

\end{itemize}

Spatial inconsistency is complementary rather than universal, it excels when models are overconfident on multi-label tasks but degrades when failures are spatially uniform.

\section{Related Work}\label{sec:related_work}

\paragraph{Bayesian and stochastic uncertainty estimation.}
MC-Dropout approximates Bayesian inference via $T$ stochastic forward passes but requires dropout layers to be present at training time, making it inapplicable to frozen checkpoints~\citep{gal2016dropout}. Bayesian weight approximations offer principled posteriors but scale poorly~\citep{graves2011practical,blundell2015weight}. Deep Ensembles achieve strong calibration at the cost of multiplying compute~\citep{lakshminarayanan2017simple}. A recent benchmark of 13 such techniques on multi-label chest X-ray data confirms no single method is universally optimal~\citep{baur2025benchmarking}. SpatialUQ imposes no training-time assumptions

\paragraph{Test-time augmentation and spatial perturbation.}
TTA treats prediction variance across augmentations as an uncertainty proxy~\citep{krizhevsky2012imagenet,wang2019aleatoric,ayhan2018test}, though strategy matters more than diversity~\citep{shanmugam2021better}. Geometric and photometric augmentations confound spatial ambiguity with color/contrast sensitivity.  SpatialUQ instead measures whether predictions are stable across spatial sub-regions a geometric property of evidence localization formalized in Lemma~\ref{lemma:bound} and compared against photometric TTA-JSD in Section~\ref{sec:photometric_perturbation}.

\paragraph{Post-hoc feature- and logit-based methods.}
Mahalanobis distance~\citep{lee2018simple}, DDU ~\citep{mukhoti2021deterministic}, energy scoring, and ODIN~\citep{liu2020energy,liang2017enhancing} operate on training-set feature distributions or logits. Conformal prediction provides distribution-free coverage guaranties but constructs prediction sets rather than ranking inputs by failure probability~\citep{angelopoulos2023conformal}. SpatialUQ produces a scalar triage score. Combining MUS-based ranking with conformal calibration is a natural future direction.

\paragraph{Uncertainty in medical imaging.}
MC-Dropout has been applied to retinopathy screening~\citep{leibig2017leveraging} and brain segmentation~\citep{roy2019bayesian}. The aleatoric/epistemic distinction was formalized by~\citet{kendall2017uncertainties} and  clinical motivation is surveyed by~\citet{begoli2019need}. Temperature scaling addresses overconfidence~\citep{guo2017calibration} and selective prediction provides the risk-coverage framework for triage~\citep{geifman2017selective}. We follow~\citet{nado2021uncertainty}, reporting ECE, AURC, E-AURC, AUPR, and FPR at fixed TPR.

\paragraph{Failure and OOD detection.}
Failure detection without ground truth has been studied in object detection~\citep{miller2019evaluating}, with likelihood-ratio scores proposed as OOD detectors~\citep{ren2019likelihood}. Our MS COCO evaluation extends SpatialUQ to detection via a multi-label Bernoulli formulation. Prior calibration work on ViT~\citep{fort2019deep,minderer2021revisiting} and globally coherent self-attention~\citep{liu2021swin} motivates our investigation of masking versus up-sampling for crop construction.

\paragraph{Divergence measures and output-space uncertainty.}
Predictive entropy~\citep{shannon1948mathematical} and mutual information~\citep{houlsby2011bayesian} are canonical output-space baselines. JSD's symmetry, boundedness in $[0,\log 2]$, and robustness to zero probability mass make it well-suited to multi-label sigmoid outputs~\citep{goodfellow2014generative,hinton2015distilling}. To our knowledge, SpatialUQ is the first to apply JSD between global and local spatial predictions as a failure detection proxy, introducing a geometrically grounded interpretation requiring no model internals.

\section{Our Approach}\label{sec:our_approach}

SpatialUQ operationalizes the spatial consistency hypothesis through a deterministic five-crop decomposition. We formalize the score for single- and multi-label settings, describe the optional linear fusion, and bound the score against per-class probability shifts. Figure~\ref{fig:pipeline} illustrates the pipeline and qualitative examples.

\begin{figure*}[!ht]
    \centering
    \begin{subfigure}[b]{0.49\textwidth}
        \centering
        \includegraphics[width=\textwidth]{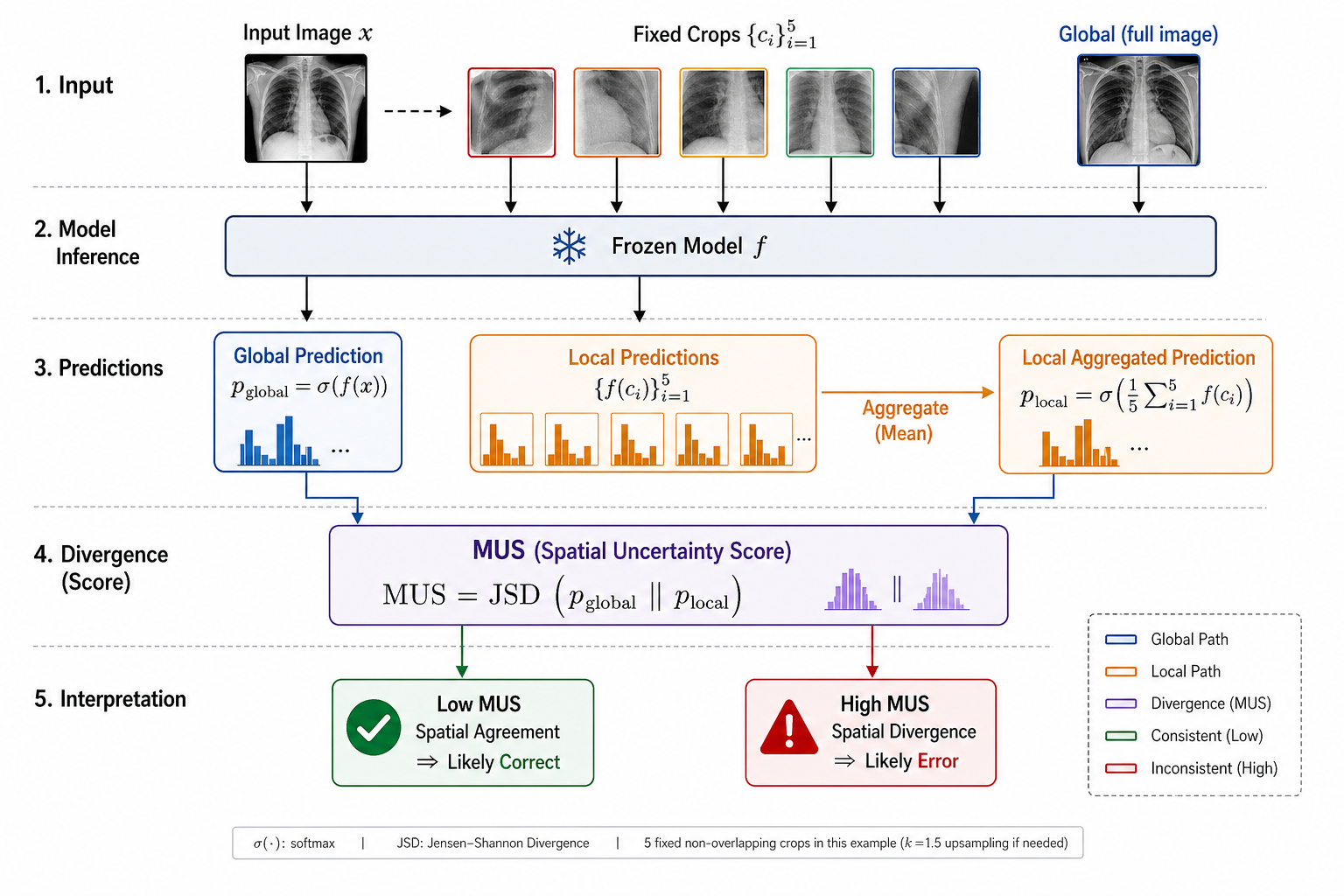}
        \caption{SpatialUQ Pipeline Overview.}
        \label{fig:pipeline_overview}
    \end{subfigure}
    \hfill
    \begin{subfigure}[b]{0.49\textwidth}
        \centering
        \includegraphics[width=\textwidth]{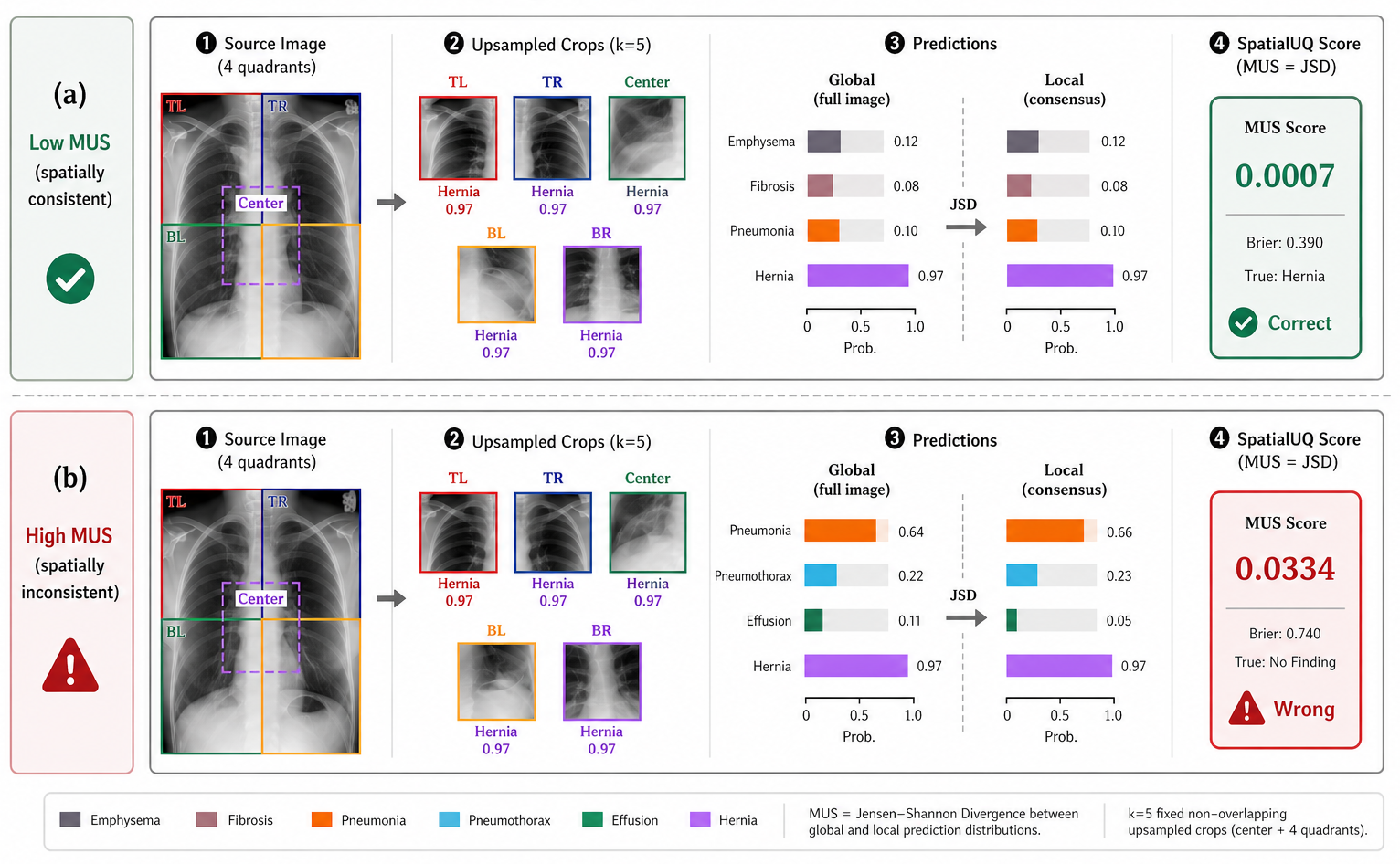}
        \caption{Qualitative Examples: Low vs.\ High MUS.}
        \label{fig:qualitative_examples}
    \end{subfigure}
    \caption{(a) The five-step inference pipeline computing JSD between global and aggregated local predictions. (b) High spatial inconsistency (MUS) successfully flags prediction errors.}
    \label{fig:pipeline}
\end{figure*}

\subsection{Spatial Decomposition}

Let $f: \mathcal{X} \to \mathcal{P}$ be a frozen classifier mapping  $x \in \mathcal{X}$ to either a probability simplex $\Delta^{C-1}$  (single-label) or $[0,1]^C$ under independent sigmoid activations  (multi-label), accessible only via output probabilities. We seek a 
deterministic scalar $s(x) \geq 0$, computable in a small constant  number of forward passes, that is positively correlated with prediction error (Brier score) at inference time.

A reliable model gives spatially consistent predictions, inconsistency signals that confidence is not anchored to stable image content. Each $224{\times}224$ input is decomposed into five fixed spatial crops, four non-overlapping quadrants and one overlapping center crop (in $[y_1, x_1, y_2, x_2]$ format): $\mathcal{C} = \{(0,0,112,112),\ (0,112,112,224),\ (112,0,224,112),\newline
    (112,112,224,224),\ (56,56,168,168)\}$, fixed across all inputs and architectures. Each crop is bi-linearly upsampled to $224{\times}224$ before passing through $f$. The aggregate local prediction is:
\begin{equation}
    p_{\text{local}}(x) = \frac{1}{5}\sum_{k=1}^{5} 
    f\!\left(\text{upsample}(x[\mathcal{C}_k])\right)
\end{equation}i
Together with $p_{\text{global}}(x) = f(x)$, this requires exactly \textbf{six forward passes} through an unmodified model.

\subsection{Uncertainty Score via Jensen-Shannon Divergence}

We quantify spatial inconsistency via JSD between $p_{\text{global}}$ and $p_{\text{local}}$ (bounded in $[0, \log 2]$, symmetric, well-defined under zero probability mass).

\paragraph{Single-label tasks (categorical JSD).}
For softmax outputs $p_{\text{global}}, p_{\text{local}} \in 
\Delta^{C-1}$, with mixture $m = \frac{1}{2}(p_{\text{global}} + 
p_{\text{local}})$:
\begin{equation}
    \begin{split}
    s_{\text{cat}}(x) &= \mathrm{JSD}(p_{\text{global}} \| p_{\text{local}}) \\
    &= \frac{1}{2}\,\mathrm{KL}(p_{\text{global}} \| m) + 
    \frac{1}{2}\,\mathrm{KL}(p_{\text{local}} \| m)
    \end{split}
\end{equation}
Applied to ImageNet-1k and MS COCO (per-class detection confidences as an 80-dimensional vector).

\paragraph{Multi-label tasks (Bernoulli JSD).}
Each class $c$ induces an independent Bernoulli distribution with 
mixture $m_c = \frac{1}{2}(p_{\text{global},c} + p_{\text{local},c})$:
\begin{equation}
    \begin{split}
    \mathrm{JSD}_c(p,q) = \frac{1}{2}\Big[&p_c \log\frac{p_c}{m_c}
    + (1-p_c)\log\frac{1-p_c}{1-m_c} \\
    &+ q_c \log\frac{q_c}{m_c} 
    + (1-q_c)\log\frac{1-q_c}{1-m_c}\Big]
    \end{split}
\end{equation}
\begin{equation}
    s_{\text{bern}}(x) = \frac{1}{C}\sum_{c=1}^{C} 
    \mathrm{JSD}_c\!\left(p_{\text{global},c},\, p_{\text{local},c}\right)
\end{equation}
Applied to NIH ChestX-ray14 ($C{=}14$), CheXpert ($C{=}10$), and VinBigData ($C{=}10$). Algorithm~\ref{alg:spatialuq} summarises the procedure: one forward pass for $p_{\text{global}}$, five for the upsampled crops ($p_1,\ldots,p_5$), average to $p_{\text{local}}$, and return $s_{\text{cat}}$ or $s_{\text{bern}}$. The pipeline is deterministic, stateless, and adds no overhead beyond the six forward passes.

\begin{algorithm}[!ht]
\SetAlgoLined\DontPrintSemicolon
\KwIn{Frozen model $f$, image $x \in \mathbb{R}^{224\times224\times3}$, task $\in \{\text{single-label, multi-label}\}$}
\KwOut{Scalar uncertainty score $s(x) \geq 0$}
$p_{\text{global}} \gets f(x)$\tcp*{forward pass 1}
\For{$k = 1$ \KwTo $5$}{
    $p_k \gets f\!\left(\text{upsample}(x[\mathcal{C}_k],\ 224{\times}224)\right)$\tcp*{forward passes 2-6}
}
$p_{\text{local}} \gets \frac{1}{5}\sum_{k=1}^{5} p_k$\;
$s \gets s_{\text{cat}}(x)$ \textbf{if} single-label \textbf{else} $s \gets s_{\text{bern}}(x)$\;
\KwRet{$s$}
\caption{SpatialUQ}\label{alg:spatialuq}
\end{algorithm}

\subsection{Linear Fusion (Supervised Extension)}
This supervised extension fits logistic weights on a small labeled hold-out; the unsupervised MUS itself needs no labels. It combines complementary output signals via logistic regression:
\begin{equation}
    \begin{split}
    s_{\text{fused}}(x) = \sigma\!\big(&\alpha_0 + \alpha_1\, s(x) + \alpha_2\, H(x) \\
    &+ \alpha_3\,(1 - \max_c p_c) + \alpha_4\, d_{\text{$\ell_1$}}(x)\big)
    \end{split}
\end{equation}
Here $d_{\text{$\ell_1$}}(x) = \frac{1}{C}\sum_c |p_{\text{global},c} - p_{\text{local},c}|$ and $\sigma(\cdot)$ is the sigmoid. Coefficients $\alpha_i$ come from 5-fold cross-validation on a small labeled hold-out: the validation set for NIH, or a random 10\% of the test set for zero-shot settings (remaining 90\% for evaluation only), preserving $f$'s zero-retraining property. Ablation is in Table~\ref{tab:ablation}.

\subsection{Information-Theoretic Bound}
\label{sec:theory}

\begin{lemma}[Spatial Disagreement Bound]
\label{lemma:bound}
For each class $c$, let $p_c \coloneqq p_{\normalfont\text{global},c}$, $q_c \coloneqq p_{\normalfont\text{local},c} \in [0,1]$. Then:
\begin{equation}
    \frac{1}{C}\sum_{c=1}^{C}\bigl|p_{\normalfont\text{global},c} - p_{\normalfont\text{local},c}\bigr| \;\leq\; \sqrt{2 \cdot s_{\normalfont\text{bern}}(x)}
\end{equation}
Here, $s_{\normalfont\text{bern}}(x)$ is in nats, matching the implementation. The per-class bound is tight in the limit $p_c, q_c \to \frac{1}{2}$, where the ratio of the two sides tends to $1$.
\end{lemma}

\begin{proof}[Proof sketch]
Applying Pinsker's inequality to $\mathrm{KL}(p_c \| m_c)$ and $\mathrm{KL}(q_c \| m_c)$ separately, where $m_c = \frac{1}{2}(p_c + q_c)$, yields $(p_c - q_c)^2 \leq 2\,\mathrm{JSD}_c$ in nats. Taking square roots and averaging over $C$ classes via Jensen's inequality (concavity of $\sqrt{\cdot}$) gives the class-averaged bound. Full derivation and tightness verification are in Appendix~\ref{app:proof}.
\end{proof}

Lemma~\ref{lemma:bound} shows a large SpatialUQ score \emph{necessarily} implies a large mean absolute shift between global and local class probabilities. This concerns probability shift only; that MUS predicts failures is empirical: the Spearman correlation between $s_{\text{bern}}$ and per-image Brier score is positive and highly significant ($p < 0.0001$) across all medical benchmarks, which licenses MUS as a failure signal. The bound is conservative under heterogeneous per-class JSD via Jensen's inequality.

\section{Experimental Setup}\label{sec:experimental_setup}

\subsection{Datasets and Tasks}\label{sec:dataset_task}

We evaluate SpatialUQ across twelve configurations spanning medical and natural image classification and object detection. 
\textbf{NIH ChestX-ray14}~\citep{wang2017chestx} (112,120 frontal chest X-rays, 14 multi-label diseases; patient-wise split \textbf{73,891/12,633/25,596} train/val/test; class-balanced BCE for $5.3\times$--$641.5\times$ imbalance) trains three architectures independently: DenseNet-121~\citep{huang2017densely}, EfficientNet-B4~\citep{tan2019efficientnet}, and ViT-B/16~\citep{dosovitskiy2020image}. CLIP ViT-B/32 and BiomedCLIP are evaluated on the same NIH test set via linear-probe transfer (frozen backbones). \textbf{CheXpert}~\citep{irvin2019chexpert} and \textbf{VinBigData}~\citep{nguyen2022vindr} serve as zero-shot transfer benchmarks for the NIH-trained DenseNet-121 (10 NIH-shared classes each; Appendix~\ref{app:label_mapping}; U-Zeros for uncertain CheXpert labels).
\textbf{ImageNet-1k} evaluates three pretrained models (EfficientNet-B4, ViT-B/16, ConvNeXt-Tiny~\citep{liu2022convnet}) zero-shot on 50,000 validation images with categorical JSD; failure is top-1 misclassification.
\textbf{MS COCO 2014}~\citep{lin2014microsoft} evaluates Faster R-CNN~\citep{ren2015faster} and RetinaNet~\citep{lin2017focal} (ResNet-50-FPN) on a 5,000-image subset (seed 42), with per-class confidence as the maximum detection score (threshold 0.1; 0.05/0.30 confirm stability) and Bernoulli JSD over 80-dimensional vectors.

\subsection{Failure Definition}\label{sec:failure_defination}
A prediction is labeled a \textit{failure} if its per-image Brier score exceeds the 75th percentile of the test-set Brier distribution, yielding a consistent 25\% failure rate across all datasets and architectures for direct cross-domain comparison without dataset-specific tuning (robustness to the 50th, 75th, and 90th percentile thresholds is confirmed in Appendix~\ref{app:failure_threshold_robustness}. Critically, the Brier score is computed against ground-truth labels and is epistemically independent of MUS, which operates solely on output probabilities.

\subsection{Training Protocol}\label{sec:training_protocol}
NIH architectures are fine-tuned from ImageNet-pretrained weights in two stages: (i) 5 epochs with frozen backbone (lr $3\times10^{-4}$, AdamW, weight decay $10^{-4}$); (ii) up to 25 epochs (patience 7) unfrozen (backbone lr $5\times10^{-6}$, head lr $5\times10^{-5}$, cosine annealing to $10^{-7}$). Loss is BCE with label smoothing ($\varepsilon=0.1$) and class-balanced \texttt{pos\_weight}, augmentation is random rotation ($\pm10^\circ$) and crop to $224\times224$ (no horizontal flip, preserving laterality), batch size 32. ImageNet and COCO models are evaluated zero-shot full details in Appendix~\ref{app:implementation}.

\subsection{Baselines}\label{sec:baseline}
We compare against 17 uncertainty methods from the same frozen model without retraining. \textbf{Output-space}: predictive entropy~\citep{shannon1948mathematical}, inverted confidence (raw and temperature-scaled), $\ell_1$ distance, and maximum inter-crop disagreement. \textbf{Crop-based}: patch variance, pairwise disagreement, and pairwise JSD. For reproducibility: patch variance is the variance of per-class probabilities across the five crops, averaged over classes, pairwise disagreement is the mean absolute difference between all pairs of crop predictions, pairwise JSD is the mean JSD between all pairs of crop predictions. \textbf{Augmentation-based}: TTA variance, TTA-JSD (10 augmentations), and MC-Dropout~\citep{gal2016dropout} (30 passes). \textbf{Feature/logit-based}: DDU~\citep{mukhoti2021deterministic}, Mahalanobis distance~\citep{lee2018simple}, and ODIN~\citep{liang2017enhancing}. \textbf{Ensemble}: mean prediction entropy, ensemble variance, and temperature-scaled entropy~\citep{lakshminarayanan2017simple} from five DenseNet-121 models. Hyperparameters follow respective publications (Appendix~\ref{app:implementation}).

\subsection{Evaluation Protocol}\label{sec:evaluation_protocol}
The primary metric is AUC for binary failure detection, we additionally report AUPR, AURC, E-AURC, and FPR at 80\% and 95\% TPR. We adopt the AURC formulation of~\citet{nado2021uncertainty} over AUGRC~\citep{zhou2024novel}, as the fixed 25\% failure rate renders the two equivalent. Calibration uses Score Calibration Error (SCE; ECE with 15 bins on min-max normalized scores). All metrics include 95\% bootstrap CIs (1,000 resamples), pairwise AUC significance uses DeLong tests and two-tailed bootstrap $p$-values (10,000 resamples). All experiments use seed 42, code and data splits are included in the supplementary material and will be released publicly upon acceptance.

\section{Results}\label{sec:results}

\subsection{NIH ChestX-ray14: In-Distribution Failure Detection}\label{sec:NIH_result}

Table~\ref{tab:nih_main} reports failure detection on NIH ChestX-ray14 ($N{=}25{,}596$) using DenseNet-121. MUS reaches an AUC of $0.784$ (95\%~CI~$[0.778, 0.789]$) with six deterministic forward passes, outperforming MC-Dropout ($0.664$, $\Delta{=}{+}0.119$, $p{<}10^{-6}$; 30 stochastic passes) while reducing wall-clock time by $5\times$ ($\approx$8 vs.\ $\approx$39 min, Appendix~\ref{app:implementation}). Although $\ell_1$ achieves a slightly higher AUC ($0.790$ vs.\ $0.784$, $\Delta{=}0.006$, $p{<}10^{-6}$), its calibration is substantially worse (SCE $0.127$ vs.\ $0.049$). Naive crop variance performs much worse ($0.610$ AUC), showing that the way crop predictions are aggregated is important. Confidence and entropy are close to chance ($0.401$ and $0.177$), reflecting systematic BCE overconfidence from label smoothing and class-balanced loss; temperature scaling does not recover the signal ($0.400$).

This advantage is specific to the overconfidence regime encountered here. On well-calibrated ImageNet classifiers, confidence alone reaches $0.913$--$0.923$ AUC, compared with $0.641$--$0.717$ for MUS, although MUS still provides complementary information in fusion (Section~\ref{sec:natural_image}). Combining MUS, entropy, confidence, and $\ell_1$ through 5-fold cross-validation on the validation set increases AUC to $0.832$, outperforming a five-member deep ensemble ($0.813$, $\Delta{=}{+}0.019$, $p{<}10^{-6}$) at a fraction of the training cost. MUS also tracks model error monotonically ($\rho{=}0.523$, $p{<}10^{-6}$), and the same behavior extends to EfficientNet-B4 and ViT-B/16 (Appendix~\ref{app:nih_arch}). Finally, spatial disagreement heatmaps show that high-MUS failures concentrate crop-level JSD in ambiguous regions (Appendix~\ref{app:sota_analysis}).

\begin{figure}[!ht]
    \centering
        \includegraphics[width=\linewidth, height=0.35\textheight, keepaspectratio]{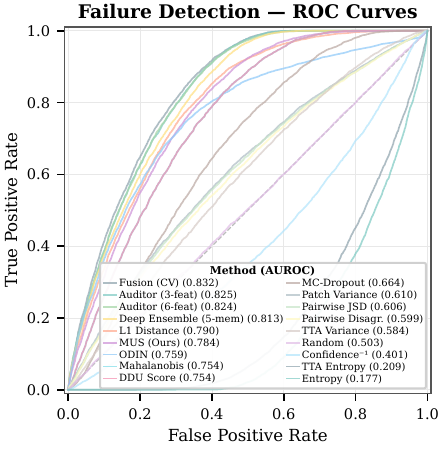}
        \captionof{figure}{Failure detection ROC curve for the NIH test 
        set, comparing SpatialUQ (MUS) against sampling-based and 
        feature-based baselines.}
        \label{fig:roc_failure_detection}
\end{figure}

\begin{table}[tb]
    \centering
    \caption{Failure detection on NIH ChestX-ray14
    ($N{=}25{,}596$, failure: Brier $>$ 75th percentile).
    $\dagger$ requires $>$6 passes. $\ddagger$ requires training
    features. Full results in Table~\ref{tab:nih_densenet_full}.}
    \label{tab:nih_main}
    \vspace{2pt}
    \footnotesize
    \addtolength{\tabcolsep}{-3pt}
    \begin{tabular}{lccc}
        \toprule
        \textbf{Method} & \textbf{AUC} & \textbf{95\% CI} & \textbf{SCE} \\
        \midrule
        \textbf{Fusion (CV)}   & \textbf{0.832} & \textbf{[0.827, 0.837]} & \textbf{0.246} \\
        Deep Ensemble$\dagger$       & 0.813          & [0.808, 0.819]          & 0.138          \\
        $\ell_1$ Distance            & 0.790          & [0.785, 0.796]          & 0.127          \\
        \textbf{MUS (Unsupervised)}  & \textbf{0.784} & \textbf{[0.778, 0.789]} & \textbf{0.049} \\
        Deep Ens + T-scaling$\dagger$& 0.774          & [0.768, 0.780]          & 0.090          \\
        ODIN                         & 0.759          & [0.752, 0.765]          & 0.484          \\
        TTA-JSD (Photometric)        & 0.755          & [0.750, 0.761]          & 0.043          \\
        DDU Score$\ddagger$          & 0.754          & [0.748, 0.760]          & 0.139          \\
        Mahalanobis$\ddagger$        & 0.754          & [0.748, 0.760]          & 0.283          \\
        Max Disagreement             & 0.744          & [0.738, 0.750]          & 0.186          \\
        MC-Dropout$\dagger$          & 0.664          & [0.658, 0.671]          & 0.121          \\
        Ensemble Variance            & 0.658          & [0.651, 0.665]          & 0.042          \\
        TTA (5 crops, variance)      & 0.610          & [0.602, 0.618]          & 0.029          \\
        TTA Variance (N=10)$\dagger$ & 0.584          & [0.576, 0.591]          & 0.089          \\
        T-Scaled Conf$^{-1}$         & 0.400          & [0.391, 0.408]          & 0.002          \\
        Confidence$^{-1}$            & 0.401          & [0.392, 0.409]          & 0.422          \\
        \bottomrule
    \end{tabular}
\end{table}

Figure~\ref{fig:roc_failure_detection} shows that MUS achieves $0.784$ AUC with six deterministic passes, outperforming MC-Dropout ($0.664$), crop-variance TTA ($0.610$), and all other TTA variants ($\leq 0.584$) without ensembles, stochastic inference, or auxiliary supervision. Fusion further improves performance to $0.832$ AUC by combining MUS with complementary signals.

\subsection{Extension to Frozen Foundation Models}\label{sec:foundation_models}

We apply the NIH protocol to frozen CLIP ViT-B/32~\citep{radford2021learning} and BiomedCLIP~\citep{zhang2023biomedclip} via linear-probe transfer. Confidence$^{-1}$ collapses on uncalibrated medical outputs ($0.095$/$0.206$ AUC), while MUS remains strong at $0.825$ ($\rho{=}0.631$) and $0.899$ ($\rho{=}0.846$, $p{<}10^{-6}$), respectively. Although $\ell_1$ has slightly higher AUC, MUS is better calibrated (SCE $0.133$ vs.\ $0.232$), and fusion reaches $0.883$/$0.925$. SpatialUQ therefore transfers across general and domain-specific foundation models without architectural access (Appendix~\ref{app:foundation_models}).

\subsection{Calibration as a Deployment Signal}\label{sec:calibration}

Calibration determines whether a threshold is clinically meaningful, AUC measures only ranking. Figure~\ref{fig:calibration} shows MUS natively calibrated ($\text{SCE}{=}0.049$) versus raw $\ell_1$ ($0.127$). A labeled hold-out lets isotonic regression calibrate $\ell_1$ ($0.006$/$0.789$) and fusion ($0.004$/$0.830$), but such labels are unavailable for frozen black-box deployment, where MUS remains the only natively calibrated discriminative option for deployment. 

\begin{figure}[!ht]
    \centering
    \includegraphics[width=\linewidth]{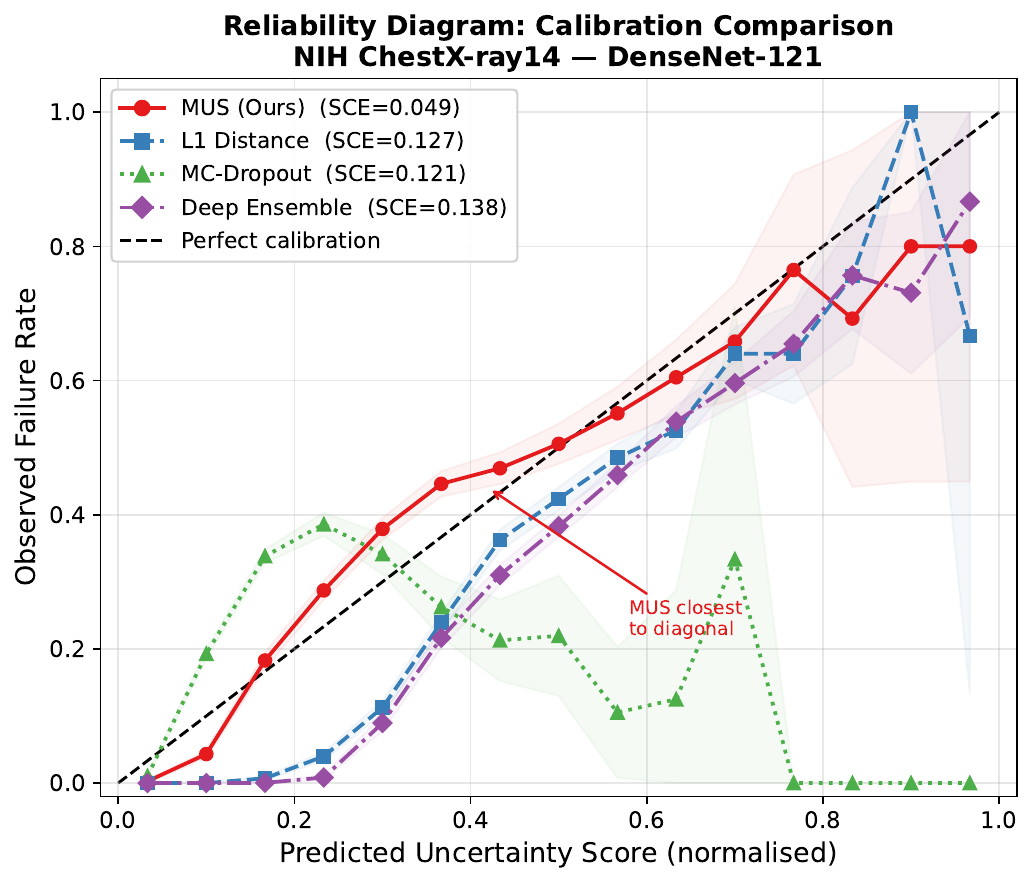}
    \caption{Reliability diagram comparing MUS against baseline uncertainty methods on NIH ChestX-ray14 (DenseNet-121).}
    \label{fig:calibration}
\end{figure}

\subsection{Ablation: Fusion Components and Spatial Design}\label{sec:ablation}

Table~\ref{tab:ablation} shows entropy is the most critical fusion component (removal: $-0.044$ AUC). MUS and $\ell_1$ add complementary gains ($-0.005$, $-0.007$), confidence negligibly ($-0.001$). Yet entropy alone reaches only $0.177$ AUC under BCE overconfidence while MUS reaches $0.784$ alone, so fusion gains come from orthogonal failure modes. Five fixed crops are optimal (Appendix~\ref{app:ablation}), nine crops gain only $+0.007$ AUC at $67\%$ additional cost, three crops lose $0.007$, and $1{,}000$ random placements reach $0.739$ (Appendix~\ref{appendix:random_crops}). These results confirm the fixed geometric arrangement provides superior spatial coverage (Appendix~\ref{app:crop_count_geometry}).

\begin{table}[!ht]
    \centering
    \small
    \caption{Leave-one-out ablation of the four-feature fusion (NIH DenseNet-121). Full fusion AUC $= 0.832$. $\Delta$: AUC drop when each feature is removed.}
    \label{tab:ablation}
    \addtolength{\tabcolsep}{-2pt}
    \begin{tabular}{lcc}
        \toprule
        \textbf{Removed Feature} & \textbf{AUC} & $\boldsymbol{\Delta}$ \\
        \midrule
        Entropy           & 0.788 & $-$0.044 \\
        $\ell_1$ Distance & 0.825 & $-$0.007 \\
        MUS               & 0.827 & $-$0.005 \\
        Confidence        & 0.831 & $-$0.001 \\
        \bottomrule
    \end{tabular}
\end{table}

\subsubsection{Spatial vs.\ Photometric Perturbation}\label{sec:photometric_perturbation}

The identical Bernoulli JSD on five photometric augmentations (brightness $\pm0.15$, contrast $\times1.2$/$\times0.8$, noise $\sigma{=}0.05$) yields $0.755$ AUC ($\rho{=}0.493$) versus MUS's $0.784$ ($\rho{=}0.523$) at the same six passes, with marginally better calibration (SCE $0.043$ vs.\ $0.049$). The AUC advantage confirms the gain comes from the perturbation's geometric structure, not the divergence measure.

\subsection{Zero-Shot Generalization to External Medical Datasets}\label{sec:zero_shot}

The NIH-trained DenseNet-121 transfers without retraining to CheXpert ($N{=}44{,}399$) and VinBigData ($N{=}15{,}000$); results are summarized in Table~\ref{tab:zeroshot}. On CheXpert, MUS achieves $0.708$ AUC versus $0.599$ for MC-Dropout ($\Delta{=}+0.109$, $p{<}10^{-6}$), with $\rho$ declining from $0.523$ to $0.298$ but retaining useful discrimination. At 20\% referral, selective prediction improves system AUC from $0.686$ to $0.748$ $+19.7$\% relative gap reduction. Under severe shift on VinBigData, however, MC-Dropout outperforms MUS ($0.764$ vs.\ $0.614$, $\Delta{=}-0.150$, $p{<}10^{-6}$), while $\rho$ collapses to $0.027$ as failures become spatially uniform and MUS loses its inconsistency signal.

\begin{table}[!ht]
\caption{Zero-shot failure detection (NIH-trained DenseNet-121, no retraining). Spearman $\rho$ = correlation between MUS and per-image Brier score. Full per-method tables in Appendix~\ref{app:zeroshot_full}.}
\label{tab:zeroshot}
\centering\small
\setlength{\tabcolsep}{3pt}
\begin{tabular}{@{}p{1.75cm} p{1.1cm} p{1.2cm} p{1.4cm} p{1.8cm}@{}}
\toprule
\textbf{Dataset} & \textbf{MUS AUC} & \textbf{Fusion AUC} & \textbf{MC-Dropout AUC} & \textbf{Spearman} $\rho$ \\
\midrule
NIH (in-dist.)         & \textbf{0.784} & \textbf{0.832} & 0.664 & 0.523 \\
CheXpert (zero-shot)   & \textbf{0.708} & \textbf{0.733} & 0.599 & 0.298 \\
VinBigData (zero-shot) & \textbf{0.614} & \textbf{0.750 }& 0.764 & 0.027 \\
\bottomrule
\end{tabular}
\end{table}

\begin{table*}[!t]
\caption{Failure detection AUC on ImageNet-1k and MS COCO 2014. 
Full secondary metrics in Appendix~\ref{app:full_metrics}.}
\label{tab:imagenet_coco}
\centering\small
\resizebox{\textwidth}{!}{
\begin{tabular}{llccccc c}
    \toprule
    \textbf{Dataset} & \textbf{Model} & \textbf{Top-1 Acc.} 
    & \textbf{MUS} & \textbf{Confidence$^{-1}$} 
    & \textbf{Entropy} & \textbf{Fusion} 
    & \textbf{Spearman} $\rho$ \\
    \midrule
    ImageNet & EfficientNet-B4 & 79.27\% & 0.717 & 0.913 
             & 0.896 & 0.917 & 0.392 \\
    ImageNet & ViT-B/16        & 81.07\% & 0.710 & 0.923 
             & 0.884 & 0.937 & 0.341 \\
    ImageNet & ConvNeXt-Tiny   & 82.13\% & 0.641 & 0.913 
             & 0.819 & 0.936 & 0.139 \\
    \midrule
    COCO     & Faster R-CNN    & -     & 0.800 & 0.602 
             & 0.884 & 0.885 & 0.613 \\
    COCO     & RetinaNet       & -     & 0.747 & 0.702 
             & 0.912 & 0.913 & 0.510 \\
    \bottomrule
\end{tabular}}
\end{table*}

\subsection{Natural Image Classification and Object Detection}\label{sec:natural_image}

Table~\ref{tab:imagenet_coco} evaluates ImageNet-1k (three architectures, 50,000 images) and MS COCO 2014 (two detectors, 5,000 images). On well-calibrated single-label classifiers, confidence already provides a strong correctness signal, making MUS complementary rather than a replacement (Section~\ref{sec:discussion}). On ImageNet, MUS alone trails Confidence$^{-1}$ ($0.641$--$0.717$ vs.\ $0.913$--$0.923$ AUC), but fusion reaches $0.917$--$0.937$, with MUS consistently contributing in leave-one-out ablations (Appendix~\ref{app:full_metrics}). This suggests that spatial consistency provides information beyond confidence even when confidence is already highly predictive. On COCO, MUS is stronger as a standalone signal ($0.800$/$0.747$ AUC; $\rho{=}0.613$/$0.510$), approaching the NIH in-distribution correlation and benefiting from the richer spatial structure of multi-label detection. Entropy remains the strongest standalone signal ($0.884$/$0.912$), while fusion reaches $0.885$/$0.913$.

\subsection{Clinical Triage Simulation}\label{sec:triage_simulations}

\paragraph{Radiologist review simulation.} We simulate triage on the NIH test set (DenseNet-121). Referring the 9.2\% highest-MUS images to human review recovers 20\% of all Brier-defined failures at 54.5\% referral precision. The threshold is tunable: precision is 58.7\% at 10\% recall (4.3\% referred), 54.5\% at 20\%, 51.4\% at 30\% (Appendix~\ref{app:clinical}). Rejecting the 20\% most uncertain images raises mean test AUC $0.808{\to}0.814$ ($\Delta{=}{+}0.006$).

\begin{figure}[!ht]
    \centering
    \includegraphics[width=0.92\linewidth]{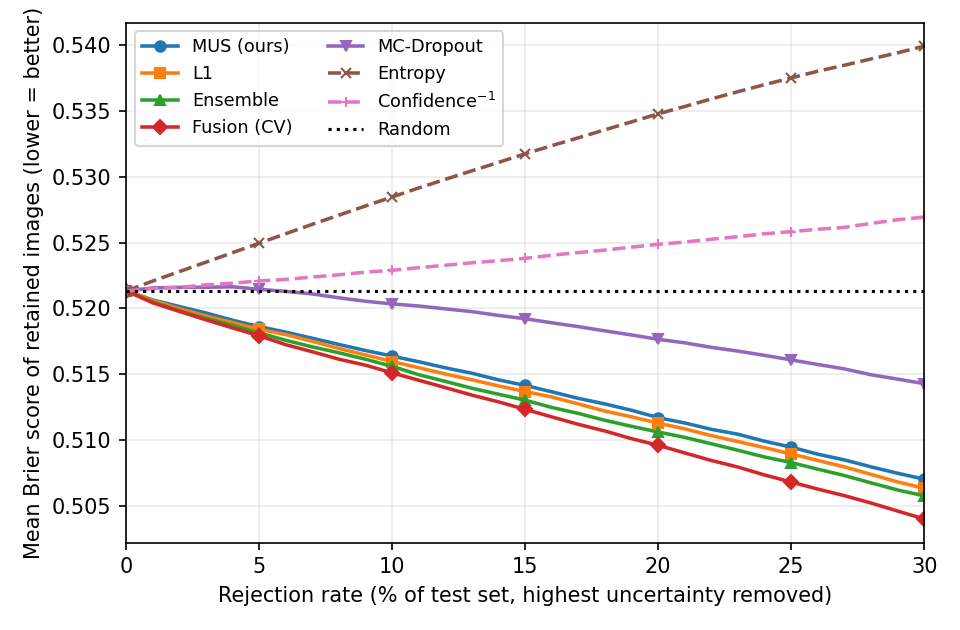}
    \caption{Mean per-image Brier score of the retained set (lower is better) versus rejection rate (NIH DenseNet-121, $N{=}25{,}596$), rejecting images in descending uncertainty order.}
    \label{fig:rejection_comparison}
\end{figure}

Figure~\ref{fig:rejection_comparison} compares rejection quality across all methods against a random baseline. MUS monotonically lowers retained Brier ($0.521{\to}0.507$ at 30\%) and beats MC-Dropout throughout. Fusion (CV) leads ($0.504$), with $\ell_1$ and ensemble ($0.506$) marginally ahead of MUS. MUS remains the best-calibrated unsupervised score above 0.78 AUC (SCE $0.049$ vs.\ $\ell_1$'s $0.127$). Entropy and confidence degrade the retained set, as expected from their near-chance AUC, since Brier defines the failure label, this restates the AUC ranking.

\begin{table*}[!t]
\caption{Failure detection on NIH ChestX-ray14 (DenseNet-121, $N{=}25{,}596$). $^\dagger$Post-hoc approximation.}
\label{tab:sota}
\centering\footnotesize
\setlength{\tabcolsep}{4pt}\renewcommand{\arraystretch}{0.95}
\begin{tabular}{@{}p{6.8cm} p{1.25cm} p{3cm} p{1.2cm} p{1cm} p{1cm} p{1cm} @{}}
\toprule
\textbf{Method} & \textbf{AUC} & \textbf{Passes} & \textbf{Internals} & \textbf{Train data} & \textbf{Retrain} & \textbf{Frozen ckpt?} \\
\midrule
\textbf{Fusion (Ours) }                      & \textbf{0.832} & \textbf{6 + logistic reg}   & \xmark & \xmark & \xmark & \cmark \\
Deep Ensemble~\citep{lakshminarayanan2017simple} & 0.813      & $5\times$model     & \xmark & \xmark & \cmark & \xmark \\
\textbf{MUS (Ours)}                        & \textbf{0.784} & \textbf{6}         & \xmark & \xmark & \xmark & \cmark \\
Deep Ensemble + T-scaling                  & 0.774          & $5\times$model+cal & \xmark & \cmark & \cmark & \xmark \\
ODIN~\citep{liang2017enhancing}              & 0.759          & 1+grad             & \cmark & \xmark & \xmark & \cmark \\
DDU$^\dagger$~\citep{mukhoti2021deterministic} & 0.754       & 1                  & \cmark & \cmark & \xmark & \xmark \\
Mahalanobis$^\dagger$~\citep{lee2018simple} & 0.754          & 1                  & \cmark & \cmark & \xmark & \xmark \\
MC-Dropout~\citep{gal2016dropout}           & 0.664          & 30                 & \xmark & \xmark & \cmark & \xmark \\
\bottomrule
\end{tabular}
\par\vspace{2pt}
{\scriptsize \textbf{Passes}: forward passes per image. \textbf{Internals}: needs gradients/activations. \textbf{Train data}: needs training-set features. \textbf{Retrain}: needs model modification. \textbf{Frozen ckpt?}: applicable to any pretrained checkpoint.\par}
\end{table*}

\paragraph{Per-class analysis.} MUS discrimination varies strongly across pathologies: it performs best on diffuse, globally visible conditions such as Pneumonia (JSD AUC $0.979$) and Edema ($0.971$), but poorly on small focal lesions such as Nodule ($0.368$), Fibrosis ($0.553$), and Mass ($0.558$), reflecting a limitation of the five-crop design. However a 16-crop extension  raises Nodule AUC to $0.589$ Appendix~\ref{app:multiscale}). Prevalence-weighting across the 14 classes yields an overall AUC of $0.83$, but masks this weakness because the three hardest classes account for only 14\% of positive findings. A focal-lesion-enriched benchmark would provide a more stringent test of MUS.

\begin{figure}[!ht]
    \centering
    \includegraphics[width=0.92\linewidth]{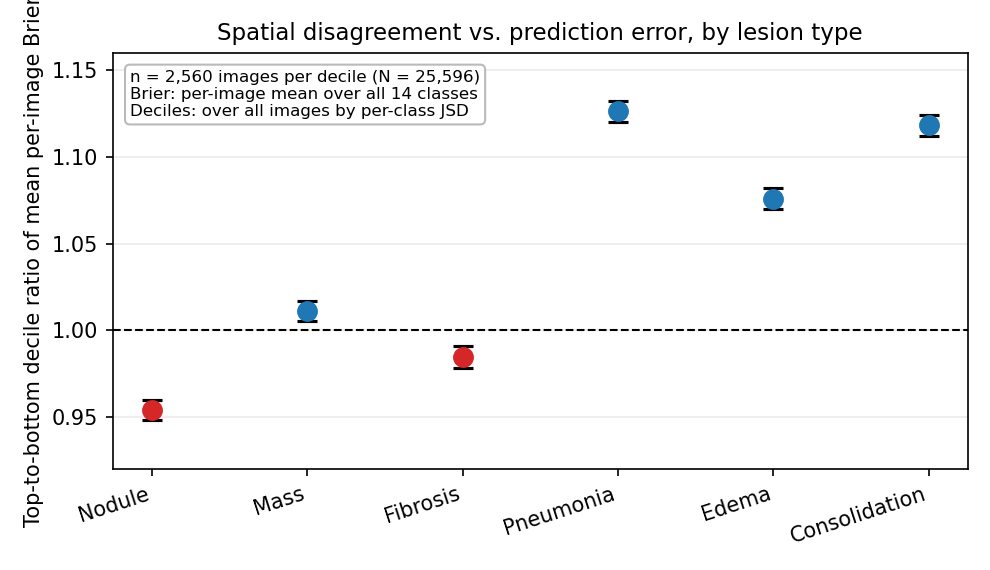}
    \caption{Top-to-bottom decile ratio of mean per-image Brier score by lesion type (NIH DenseNet-121, $N{=}25{,}596$). Error bars: 95\% bootstrap CIs; n = 2,560 images per decile.}
    \label{fig:nodule_mechanism}
\end{figure}

Figure~\ref{fig:nodule_mechanism} shows the top-to-bottom decile Brier ratio (95\% CIs): $1.08$--$1.13$ for diffuse classes but $0.95$ [0.948, 0.960] for Nodule, indicating that high nodule disagreement concentrates on better-predicted images. Mass ($1.01$) and Fibrosis ($0.99$) show no meaningful coupling. Thus, disagreement does not track error for focal lesions, although this does not imply per-crop correctness. The limitation concerns the score rather than classifier quality: raw Nodule AUC ($0.740$) is near the 14-class mean ($0.7885$, Appendix~\ref{app:backbone}). Full curves, referral statistics, and fairness analyzes are in Appendix~\ref{app:clinical} where gender parity is near-perfect ($0.784$ vs.\ $0.784$).

\section{Discussion}\label{sec:discussion}
A practitioner summary of when to use MUS is provided in Appendix~\ref{app:whentouse} (Table~\ref{tab:whentouse}).

\paragraph{Complementarity and Regime Dependence.}
MUS is most useful in overconfident multi-label settings where confidence is weak (NIH DenseNet-121: AUC $0.401$). On well-calibrated single-label tasks such as ImageNet, confidence dominates and MUS is complementary, with fusion combining the two signals.

\paragraph{Uncertainty type and OOD positioning.}
MUS measures whether model confidence is spatially reproducible, which we view as an epistemic signal reflecting learned spatial grounding rather than irreducible label noise. Its correlation with model quality also provides a pre-deployment diagnostic: $\rho\lesssim0.15$ flags severe shift, as observed on VinBigData ($\rho=0.027$), before full test-set evaluation.

\paragraph{Why MUS over $\ell_1$?}
Although $\ell_1$ slightly exceeds MUS in AUC ($0.790$ vs.\ $0.784$, $\Delta=0.006$; $p<10^{-6}$), its SCE is substantially worse ($0.127$ vs.\ $0.049$). Under shift, MUS also performs better on VinBigData ($0.614$ vs.\ $0.583$), consistent with the robustness of bounded, symmetric JSD. Lemma~\ref{lemma:bound} establishes the connection between MUS and $\ell_1$, but the empirical advantage of MUS lies in its native calibration and shift robustness. Isotonic regression can calibrate $\ell_1$ (SCE $0.006$), but requires labeled hold-out data unavailable in frozen deployment. Fusion's higher SCE ($0.246$) is instead driven by the uncalibrated entropy component.

\paragraph{Limitations.}
SpatialUQ has five main limitations. (i) Severe domain shift can collapse the inconsistency signal (VinBigData, $\rho=0.027$). (ii) Five fixed crops limit sensitivity to small focal lesions (Nodule JSD AUC $0.368$; 16-crop extension: $0.589$). (iii) ViT global attention compresses MUS's dynamic range, with masking recovering $+0.004$ AUC (Appendix~\ref{app:vit}). (iv) Fusion requires a labeled hold-out set. (v) Generative VLMs with global pooling can weaken the signal (Appendix~\ref{app:biovilt}). The $\rho\lesssim0.15$ diagnostic flags the observed failure regimes while remaining positive for discriminative settings (NIH $0.523$, CheXpert $0.298$); when flagged, MC-Dropout or ensembles are preferable, with MUS retained as a fusion component. Thus, the primary operating regime is low-cost failure detection for overconfident multi-label classifiers under moderate shift when retraining is infeasible. Generalization beyond radiography and RGB natural images remains untested.

\paragraph{Comparison with State-of-the-Art.}
Table~\ref{tab:sota} compares SpatialUQ with established UQ methods on NIH DenseNet-121. MUS achieves $0.784$ AUC fully post-hoc, without retraining, training data, or architectural modifications, while competitive baselines require training-time structure (MC-Dropout, ensembles) or dataset-dependent statistics (DDU, Mahalanobis). Extended analysis appears in Appendix~\ref{app:sota_analysis}. Temperature scaling improves ensemble SCE from $0.138$ to $0.090$ but reduces AUC from $0.813$ to $0.774$, illustrating the trade-off between calibration and discrimination. Fusion reaches $0.832$ AUC by adding the spatially grounded MUS signal, whose native SCE is $0.049$.

\section{Final Remarks}\label{sec:final_remarks}

SpatialUQ provides a practical operating point for inference-time failure detection: six deterministic forward passes using output probabilities alone, with no access to model internals, labels or retraining. Spatial inconsistency provides a geometrically grounded signal that scales with model quality, from $0.784$ AUC on DenseNet-121 to $0.899$ on BiomedCLIP ($\rho=0.846$). The $\rho\lesssim0.15$ diagnostic further provides a simple pre-deployment check for the observed distribution-shift failure regimes. Future work includes saliency-guided crop placement, aggregation adapted to global-attention architectures, and broader validation on MIMIC-CXR with temporal shift, population-level AURC, and intersectional subgroup analyzes. Generalization to CT, MRI, and ultrasound also remains to be established.

\subsection*{AI use statement}

In this work, we used generative AI tools for writing assistance, including drafting, editing, and formatting of manuscript text, and for verifying the consistency of reported numbers against the experimental code. We have not used generative AI tools for generating experimental results, fabricating data, or conducting literature surveys, and all AI-assisted content was reviewed by the authors. We take responsibility for the final content of this work, including text, claims, or artifacts produced with the aid of generative AI.

\bibliographystyle{plainnat} 
\bibliography{references}

%%%%%%%%%%%%%%%%%%%%%%%%%%%%%%%%%%%%%%%%%%%%%%%%%%%%%%%%%%%%
\section*{Checklist}

\begin{enumerate}

  \item For all models and algorithms presented, check if you include:
  \begin{enumerate}
    \item A clear description of the mathematical setting, assumptions, algorithm, and/or model. [Yes] The spatial decomposition, Bernoulli/categorical JSD scores, and the logistic fusion are defined in Section~\ref{sec:our_approach}, with the full procedure in Algorithm~\ref{alg:spatialuq}.
    \item An analysis of the properties and complexity (time, space, sample size) of any algorithm. [Yes] SpatialUQ needs exactly six forward passes per image with no stored state (Section~\ref{sec:our_approach}); wall-clock timings are in Appendix~\ref{app:implementation}. Lemma~\ref{lemma:bound} gives the information-theoretic property of the score.
    \item (Optional) Anonymized Source code, with specification of all dependencies, including external libraries. [Yes] Code and data splits are included in the supplementary material, with dependencies listed in the repository requirements file.
  \end{enumerate}

  \item For any theoretical claim, check if you include:
  \begin{enumerate}
    \item Statements of the full set of assumptions of all theoretical results. [Yes] Lemma~\ref{lemma:bound} states its assumptions ($p_c, q_c \in [0,1]$, JSD in nats) directly.
    \item Complete proofs of all theoretical results. [Yes] The full proof is in Appendix~\ref{app:proof}, with a proof sketch in Section~\ref{sec:theory}.
    \item Clear explanations of any assumptions. [Yes] Section~\ref{sec:theory} explains that the bound concerns probability shift only and that the link to failure detection is empirical (Spearman correlation).
  \end{enumerate}

  \item For all figures and tables that present empirical results, check if you include:
  \begin{enumerate}
    \item The code, data, and instructions needed to reproduce the main experimental results (either in the supplemental material or as a URL). [Yes] Code, data splits, and fixed random seeds (seed 42) are in the supplementary material.
    \item All the training details (e.g., data splits, hyperparameters, how they were chosen). [Yes] Section~\ref{sec:training_protocol} and Appendix~\ref{app:implementation} give the two-stage protocol, learning rates, augmentation, and the 73,891/12,633/25,596 patient-wise split.
    \item A clear definition of the specific measure or statistics and error bars (e.g., with respect to the random seed after running experiments multiple times). [Yes] Failure is defined as Brier score above the 75th percentile (Section~\ref{sec:failure_defination}); all AUCs carry 95\% bootstrap CIs (1,000 resamples) and pairwise significance uses DeLong tests with 10,000-resample bootstrap $p$-values.
    \item A description of the computing infrastructure used. (e.g., type of GPUs, internal cluster, or cloud provider). [Yes] All NIH experiments ran on a single NVIDIA Tesla T4 GPU (Appendix~\ref{app:implementation}).
  \end{enumerate}

  \item If you are using existing assets (e.g., code, data, models) or curating/releasing new assets, check if you include:
  \begin{enumerate}
    \item Citations of the creator If your work uses existing assets. [Yes] NIH ChestX-ray14~\citep{wang2017chestx}, CheXpert~\citep{irvin2019chexpert}, VinBigData~\citep{nguyen2022vindr}, MS COCO~\citep{lin2014microsoft}, CLIP~\citep{radford2021learning}, and BiomedCLIP~\citep{zhang2023biomedclip} are cited, as are the DenseNet~\citep{huang2017densely}, EfficientNet~\citep{tan2019efficientnet}, and ViT~\citep{dosovitskiy2020image} backbones.
    \item The license information of the assets, if applicable. [Yes] All datasets and pretrained models are publicly released research assets used under their published terms.
    \item New assets either in the supplemental material or as a URL, if applicable. [Yes] The SpatialUQ implementation notebooks are available at this link \url{https://huggingface.co/datasets/kawsher11/SpatialUQ}.

    \item Information about consent from data providers/curators. [Not Applicable] All datasets are pre-existing public benchmarks; no new data was collected.
    \item Discussion of sensible content if applicable, e.g., personally identifiable information or offensive content. [Not Applicable] The chest X-ray benchmarks are de-identified; the fairness audit in Appendix~\ref{app:clinical} uses only aggregate age/gender metadata.
  \end{enumerate}

  \item If you used crowdsourcing or conducted research with human subjects, check if you include:
  \begin{enumerate}
    \item The full text of instructions given to participants and screenshots. [Not Applicable]
    \item Descriptions of potential participant risks, with links to Institutional Review Board (IRB) approvals if applicable. [Not Applicable]
    \item The estimated hourly wage paid to participants and the total amount spent on participant compensation. [Not Applicable]
  \end{enumerate}
  No crowdsourcing or human-subjects research was conducted; the triage study in Section~\ref{sec:triage_simulations} is a retrospective simulation on de-identified benchmark data.

\end{enumerate}

\clearpage
\appendix
\thispagestyle{empty}
\onecolumn
\aistatstitle{SpatialUQ: Post-Hoc Uncertainty Quantification from Spatial Consistency in Black-Box Vision Models:\\Supplementary Materials}

\section{Proof of Lemma~\ref{lemma:bound}}\label{app:proof}

We prove the spatial disagreement bound in two steps. All logarithms are natural (JSD in nats) unless stated otherwise.

\paragraph{Step 1: Per-class bound via Pinsker's inequality.}
For Bernoulli distributions with parameters $p_c, q_c \in [0,1]$, let 
$m_c = \frac{1}{2}(p_c + q_c)$. The JSD decomposes as:
\begin{equation}
    \mathrm{JSD}_c(p_c \| q_c) = 
    \frac{1}{2}\,\mathrm{KL}(p_c \| m_c) + \frac{1}{2}\,\mathrm{KL}(q_c \| m_c)
\end{equation}
For Bernoulli distributions, the total variation distance between 
parameters $p_c$ and $m_c$ reduces to $\mathrm{TV}(p_c, m_c) = 
|p_c - m_c| = \frac{1}{2}|p_c - q_c|$. Pinsker's inequality 
$\mathrm{TV}(P,Q)^2 \leq \frac{1}{2}\mathrm{KL}(P\|Q)$ then gives:
\begin{equation}
    \frac{1}{4}(p_c - q_c)^2 \leq \frac{1}{2}\,\mathrm{KL}(p_c \| m_c)
    \quad\Longrightarrow\quad
    (p_c - q_c)^2 \leq 2\,\mathrm{KL}(p_c \| m_c)
\end{equation}
Applying the same argument symmetrically to $\mathrm{KL}(q_c \| m_c)$:
\begin{equation}
    (p_c - q_c)^2 \leq 2\,\mathrm{KL}(q_c \| m_c)
\end{equation}

Adding the two inequalities above:
\begin{equation}
    2(p_c - q_c)^2 \leq 2\bigl(\mathrm{KL}(p_c \| m_c) + 
    \mathrm{KL}(q_c \| m_c)\bigr) = 4\,\mathrm{JSD}_c(p_c \| q_c)
\end{equation}
Dividing by 2 and taking square roots:
\begin{equation}
    |p_c - q_c| \leq \sqrt{2\,\mathrm{JSD}_c(p_c \| q_c)}
    \qquad\text{(in nats)}
    \label{eq:perclass}
\end{equation}

\paragraph{Step 2: Aggregation over classes via Jensen's inequality.}
Averaging Eq.~(\ref{eq:perclass}) over all $C$ classes:
\begin{equation}
    \frac{1}{C}\sum_{c=1}^{C}|p_c - q_c| 
    \leq \frac{1}{C}\sum_{c=1}^{C}\sqrt{2\,\mathrm{JSD}_c}
\end{equation}
Since $\sqrt{\cdot}$ is concave, Jensen's inequality gives:
\begin{equation}
    \frac{1}{C}\sum_{c=1}^{C}\sqrt{2\,\mathrm{JSD}_c}
    \leq \sqrt{\frac{2}{C}\sum_{c=1}^{C}\mathrm{JSD}_c}
    = \sqrt{2\,s_{\text{bern}}(x)}
    \qquad\text{(in nats)}
\end{equation}
which is the complete bound:
\begin{equation}
    \frac{1}{C}\sum_{c=1}^{C}
    \bigl|p_{\text{global},c} - p_{\text{local},c}\bigr|
    \;\leq\; \sqrt{2 \cdot s_{\text{bern}}(x)}
    \qquad\blacksquare
\end{equation}

\paragraph{Verification of tightness.}
The per-class bound in Eq.~(\ref{eq:perclass}) is tight in the limit $p_c, q_c \to \frac{1}{2}$. Setting $p_c = \frac{1}{2} + \varepsilon$, $q_c = \frac{1}{2} - \varepsilon$ (so $m_c = \frac{1}{2}$), a second-order expansion gives $\mathrm{KL}(p_c \| m_c) = \mathrm{KL}(q_c \| m_c) = 2\varepsilon^2 + o(\varepsilon^2)$, hence $\mathrm{JSD}_c = 2\varepsilon^2 + o(\varepsilon^2)$ and
\begin{equation}
    \text{RHS} = \sqrt{2\,\mathrm{JSD}_c} = 2\varepsilon + o(\varepsilon) = |p_c - q_c| + o(\varepsilon) = \text{LHS} + o(\varepsilon),
\end{equation}
so the ratio of the two sides tends to $1$ as $\varepsilon \to 0$.

% ============================================================
% COMPLETE APPENDIX — SPATIALUQ RESULTS
% All tables referenced from main body Results section.
% ============================================================

% ============================================================
\section{Backbone Classification Performance}\label{app:backbone}
% ============================================================
Table~\ref{tab:backbone_densenet} reports per-class AUC-ROC, Brier score, and average precision for the DenseNet-121 backbone (7.0\,M parameters) trained on NIH ChestX-ray14, selected at epoch 28 (validation loss of 2.0564). The model achieves a mean AUC of 0.7885, with discriminative performance ranging from 0.6984 (Infiltration) to 0.8822 (Emphysema). Calibration and average precision exhibit pronounced class-level disparity, largely attributable to severe label imbalance. These results serve as the reference operating point for the subsequent failure-detection analysis.

\begin{table}[!ht]
\caption{NIH ChestX-ray14 backbone classification performance -
DenseNet-121 (7.0\,M parameters, best checkpoint epoch 28,
validation loss 2.0564). Mean AUC of 0.7885.}
\label{tab:backbone_densenet}
\centering\small
\begin{tabular}{lccc}
\toprule
\textbf{Class} & \textbf{AUC} & \textbf{Brier} & \textbf{AP} \\
\midrule
Atelectasis        & 0.7628 & 0.3189 & 0.3272 \\
Cardiomegaly       & 0.8746 & 0.5565 & 0.3350 \\
Effusion           & 0.8192 & 0.3081 & 0.4907 \\
Infiltration       & 0.6984 & 0.3161 & 0.3969 \\
Mass               & 0.7965 & 0.3576 & 0.2915 \\
Nodule             & 0.7400 & 0.3869 & 0.2191 \\
Pneumonia          & 0.7020 & 0.7346 & 0.0463 \\
Pneumothorax       & 0.8471 & 0.4472 & 0.3864 \\
Consolidation      & 0.7401 & 0.5084 & 0.1562 \\
Edema              & 0.8368 & 0.6476 & 0.1494 \\
Emphysema          & 0.8822 & 0.6111 & 0.3595 \\
Fibrosis           & 0.8095 & 0.6404 & 0.0888 \\
Pleural\_Thickening& 0.7569 & 0.5222 & 0.1245 \\
Hernia             & 0.7733 & 0.9428 & 0.0434 \\
\midrule
\textbf{Mean}      & \textbf{0.7885} & - & - \\
\bottomrule
\end{tabular}
\end{table}

Table~\ref{tab:backbone_effb4} reports per-class AUC-ROC, Brier score, and average precision for the EfficientNet-B4 backbone (17.6\,M parameters) on NIH ChestX-ray14, selected at epoch 29 (validation loss 2.0793). The model achieves a mean AUC of 0.7347, trailing DenseNet-121 by 0.054 points despite nearly $2.5\times$ more parameters, with per-class AUC ranging from 0.6589 (Pneumonia) to 0.8093 (Edema). Calibration and average precision mirror the class-imbalance patterns observed in the DenseNet-121 baseline. Both backbones are subsequently treated as frozen black-boxes for post-hoc uncertainty evaluation.

\begin{table}[!ht]
\caption{NIH ChestX-ray14 backbone classification performance -
EfficientNet-B4 (17.6\,M parameters, best checkpoint epoch 29,
validation loss 2.0793). Mean AUC = 0.7347.}
\label{tab:backbone_effb4}
\centering\small
\begin{tabular}{lccc}
\toprule
\textbf{Class} & \textbf{AUC} & \textbf{Brier} & \textbf{AP} \\
\midrule
Atelectasis        & 0.7192 & 0.3321 & 0.2732 \\
Cardiomegaly       & 0.7952 & 0.5746 & 0.2118 \\
Effusion           & 0.7740 & 0.3250 & 0.4189 \\
Infiltration       & 0.6790 & 0.3102 & 0.3733 \\
Mass               & 0.7002 & 0.4114 & 0.1591 \\
Nodule             & 0.6941 & 0.4039 & 0.1520 \\
Pneumonia          & 0.6589 & 0.7389 & 0.0375 \\
Pneumothorax       & 0.8050 & 0.4824 & 0.3031 \\
Consolidation      & 0.7022 & 0.5199 & 0.1299 \\
Edema              & 0.8093 & 0.6440 & 0.1115 \\
Emphysema          & 0.7977 & 0.6235 & 0.1405 \\
Fibrosis           & 0.7547 & 0.6533 & 0.0575 \\
Pleural\_Thickening& 0.7201 & 0.5370 & 0.0922 \\
Hernia             & 0.6763 & 0.9437 & 0.0096 \\
\midrule
\textbf{Mean}      & \textbf{0.7347} & - & - \\
\bottomrule
\end{tabular}
\end{table}

Table~\ref{tab:backbone_vit} reports per-class AUC-ROC, Brier score, and average precision for the ViT-B/16 backbone (85.8,M parameters) on NIH ChestX-ray14, with early stopping triggered at epoch 17 and the best checkpoint selected at epoch 9. Despite its substantially larger capacity, the model achieves a mean AUC of 0.7921, marginally exceeding DenseNet-121 (0.7885) and EfficientNet-B4 (0.7347), with per-class AUC ranging from 0.6978 (Infiltration) to 0.8868 (Cardiomegaly). Calibration patterns remain consistent with the convolutional backbones, with Hernia (Brier = 0.9413) and Pneumonia (0.7273) exhibiting the poorest posterior estimates. This backbone is of particular interest for examining the relationship between global self-attention and spatial prediction inconsistency in the failure-detection analysis.

\begin{table}[!ht]
\caption{NIH ChestX-ray14 backbone classification performance -
ViT-B/16 (85.8\,M parameters, early stopping at epoch 17,
best validation loss at epoch 9). Mean AUC = 0.7921.}
\label{tab:backbone_vit}
\centering\small
\begin{tabular}{lccc}
\toprule
\textbf{Class} & \textbf{AUC} & \textbf{Brier} & \textbf{AP} \\
\midrule
Atelectasis        & 0.7689 & 0.2882 & 0.3399 \\
Cardiomegaly       & 0.8868 & 0.5685 & 0.3500 \\
Effusion           & 0.8257 & 0.3037 & 0.5059 \\
Infiltration       & 0.6978 & 0.3029 & 0.3983 \\
Mass               & 0.8157 & 0.3817 & 0.3176 \\
Nodule             & 0.7502 & 0.3712 & 0.2230 \\
Pneumonia          & 0.7111 & 0.7273 & 0.0555 \\
Pneumothorax       & 0.8587 & 0.4469 & 0.4263 \\
Consolidation      & 0.7521 & 0.4987 & 0.1686 \\
Edema              & 0.8359 & 0.6364 & 0.1653 \\
Emphysema          & 0.8854 & 0.5791 & 0.4034 \\
Fibrosis           & 0.7938 & 0.6235 & 0.0857 \\
Pleural\_Thickening& 0.7622 & 0.5380 & 0.1367 \\
Hernia             & 0.7446 & 0.9413 & 0.0347 \\
\midrule
\textbf{Mean}      & \textbf{0.7921} & - & - \\
\bottomrule
\end{tabular}
\end{table}

% ============================================================
\section{NIH ChestX-ray14: Full Failure Detection Tables}\label{app:nih_arch}
% ============================================================

\subsection{DenseNet-121: All Methods Comparison}

The results in Table~\ref{tab:nih_densenet_full} evaluate failure detection performance on the NIH ChestX-ray14 test set using a DenseNet-121 backbone. Failure is defined as predictions with Brier score exceeding the 75th percentile (threshold = 0.5607, yielding a 25\% failure rate). We report area under the ROC curve (AUC) with 95\% bootstrap confidence intervals (1,000 resamples) and Score Calibration Error (SCE)the expected calibration error applied to min–max normalized scores using 15 bins. Among all methods, the proposed MUS achieves the lowest SCE (0.0487), indicating well-calibrated failure probabilities, while its AUC (0.7837) is competitive. The highest AUC is obtained by Fusion CV (0.8319), which combines multiple uncertainty signals, followed closely by the two Auditor variants (0.8250, 0.8242). However, these top AUC methods exhibit substantially higher SCE (0.2464, 0.2408, 0.2363) than MUS. Deep Ensemble (5 members) yields strong AUC (0.8130) and moderate SCE (0.1379), whereas MUS strikes a favorable trade-off with near-perfect calibration. Notably, simple post-hoc baselines like T-Scaled Confidence$^{-1}$ achieve excellent SCE (0.0015) but near-random AUC (0.4000), highlighting that calibration alone is insufficient for failure detection. Conversely, Entropy performs poorly on both metrics (AUC 0.1774, SCE 0.4575). Methods marked with $\dagger$ require more than six forward passes or training-time modifications; $\ddagger$ indicates access to training-set features. The table confirms that no single method dominates all criteria, but MUS provides the best calibration among practically viable approaches. 

\noindent\textbf{Statistical tests (DeLong, two-tailed):} MUS vs.\ MC-Dropout: $\Delta\mathrm{AUC} = +0.1194$, $p < 10^{-6}$. MUS vs.\ $\ell_1$ Distance: $\Delta\mathrm{AUC} = -0.0064$, $p < 10^{-6}$ ($\ell_1$ marginally superior). Fusion vs.\ Deep Ensemble: $\Delta\mathrm{AUC} = +0.0189$, $p < 10^{-6}$. Spearman correlation between MUS and Brier score: $\rho = 0.5225$, $p < 10^{-6}$ (Pearson $r = 0.4983$).

\begin{table}[!ht]
\caption{NIH ChestX-ray14 failure detection DenseNet-121, full 16-method comparison ($N = 25{,}596$). Brier threshold = 0.5607 (75th percentile, 25\% failure rate). SCE = Score Calibration Error (ECE applied to min-max normalised scores, 15 bins). $\dagger$ requires $>$6 forward passes or training-time modification. $\ddagger$ requires access to training-set features.}
\label{tab:nih_densenet_full}
\centering\small
\setlength{\tabcolsep}{4pt}
\begin{tabular}{@{}p{2.3cm} p{1.25cm} p{2.3cm} p{1.25cm}@{}}
\toprule
\textbf{Method} & \textbf{AUC} & \textbf{95\% CI} & \textbf{SCE} \\
\midrule
\textbf{Fusion CV (MUS+\allowbreak Ent+\allowbreak Conf+\allowbreak $\ell_1$) }    & \textbf{0.8319} & [0.8268, 0.8368] & \textbf{0.2464 }\\
Auditor (3-feat MLP)$\dagger$   & 0.8250 & [0.8201, 0.8299] & 0.2408 \\
Auditor (6-feat MLP)$\dagger$   & 0.8242 & [0.8189, 0.8292] & 0.2363 \\
Deep Ensemble (5-member)$\dagger$& 0.8130 & [0.8077, 0.8187] & 0.1379 \\
$\ell_1$ Distance                     & 0.7902 & [0.7848, 0.7960] & 0.1270 \\
\textbf{MUS (Ours)}             & \textbf{0.7837} & [0.7778, 0.7893] & \textbf{0.0487} \\
Deep Ensemble + T-scaling$\dagger$ &  0.7741 & [0.7680, 0.7801] & 0.0897 \\
ODIN                            & 0.7588 & [0.7518, 0.7653] & 0.4844 \\
TTA-JSD — photometric (same formula, 5 augs)   &    0.7549 & [0.7492, 0.7607] & 0.0430\\
DDU Score$\ddagger$             & 0.7540 & [0.7476, 0.7599] & 0.1386 \\
Mahalanobis$\ddagger$           & 0.7540 & [0.7480, 0.7596] & 0.2828 \\
Max Disagreement                & 0.7440 & [0.7376, 0.7502] & 0.1861 \\
MC-Dropout (T=30)$\dagger$      & 0.6643 & [0.6575, 0.6710] & 0.1211 \\
Ensemble Variance$\dagger$      & 0.6579 & [0.6505, 0.6648] & 0.0419 \\
TTA (5 crops, patch variance)         & 0.6103 & [0.6022, 0.6178] & 0.0290 \\
Pairwise JSD                    & 0.6057 & [0.5979, 0.6133] & 0.0335 \\
Pairwise Disagreement           & 0.5990 & [0.5913, 0.6073] & 0.2436 \\
TTA Variance (N=10)$\dagger$    & 0.5837 & [0.5759, 0.5914] & 0.0889 \\
Confidence$^{-1}$               & 0.4011 & [0.3922, 0.4091] & 0.4218 \\
T-Scaled Conf$^{-1}$            & 0.4000 & [0.3910, 0.4081] & 0.0015 \\
Entropy                         & 0.1774 & [0.1727, 0.1821] & 0.4575 \\
\bottomrule
\end{tabular}
\end{table}

Table~\ref{tab:diffrent_seed_mus_auc} demonstrates that MUS exhibits strong robustness to random initialization, with AUC values remaining tightly clustered across different seeds (0.7826–0.7902), indicating low variance and stable behavior. The fluctuation range (0.0076) is minimal, suggesting that MUS is not sensitive to stochastic training effects and generalizes consistently. In comparison, $\ell_1$ Distance achieves slightly higher absolute performance across all seeds but follows a similar stability pattern, while MC-Dropout and TTA Variance show modest variability and consistently lower AUC. Notably, confidence-based and entropy-based methods remain significantly weaker and exhibit inconsistent behavior across seeds, highlighting their limited reliability. Pairwise JSD also shows mild degradation for certain seeds, suggesting sensitivity to randomness in feature representations. Overall, these results reinforce that MUS provides a reliable and reproducible uncertainty signal, maintaining competitive performance with minimal seed-induced variance an important property for deployment in high-stakes settings.

\begin{table*}[!ht]
\caption{Seed MUS: Comparison across different random seeds for DenseNet121 on NIH ChestX-ray14.}
\label{tab:diffrent_seed_mus_auc}
\centering
\footnotesize
\setlength{\tabcolsep}{0pt} % Set to 0 to let tabular* manage spacing
\begin{tabular*}{\linewidth}{@{\extracolsep{\fill}} lccccccc}
\toprule
\textbf{Seed} & \textbf{MUS AUC} & \textbf{$\ell_1$ Dist.} & \textbf{MC-Drop} & \textbf{Conf$^{-1}$} & \textbf{TTA Var} & \textbf{Entr.} & \textbf{Pair JSD} \\
\midrule
42   & 0.7837 & 0.7902 & 0.6643 & 0.4011 & 0.5837 & 0.1774 & 0.6057 \\
186  & 0.7866 & 0.7959 & 0.6466 & 0.5802 & 0.5723 & 0.1748 & 0.6023 \\
456  & 0.7902 & 0.7976 & 0.6532 & 0.5816 & 0.5856 & 0.1672 & 0.5709 \\
789  & 0.7862 & 0.7914 & 0.6498 & 0.5788 & 0.5797 & 0.1681 & 0.5806 \\
1011 & 0.7826 & 0.7844 & 0.6536 & 0.5737 & 0.5848 & 0.1688 & 0.5660 \\
\midrule
\textbf{Mean} 
& \textbf{0.7859} & 0.7919 & 0.6535 & 0.5431 & 0.5812 & 0.1713 & 0.5851 \\
\textbf{$\pm$ Std} 
& \textbf{0.0026} & 0.0047 & 0.0060 & 0.0700 & 0.0050 & 0.0040 & 0.0160 \\
\bottomrule
\end{tabular*}
\end{table*}

\subsection{Foundation Model Evaluation with Linear Probes}\label{app:foundation_models}

To assess SpatialUQ's transferability and scaling with model quality, we evaluate CLIP ViT-B/32 and BiomedCLIP on the NIH ChestX-ray14 test set with a linear probe trained on NIH train features (backbone frozen; no fine-tuning). Labels are used only to train the probe and for evaluation. Table~\ref{tab:zero_shot_results} reports failure detection AUC (75th percentile Brier threshold, 25\% failure rate), 95\% bootstrap CIs, and raw Score Calibration Error (SCE).

\begin{table*}[!ht]
\centering
\caption{Failure detection on NIH ChestX-ray14 with frozen foundation models plus linear probe (25,596 images). MUS requires only output probabilities; $L_1$ and Fusion use the same six forward passes. Fusion is a logistic regression trained on a small labeled hold-out (10\% of the test set, remaining 90\% for evaluation).}
\label{tab:zero_shot_results}
\small
\begin{tabular}{lcccc}
\toprule
\multirow{2}{*}{Method} & \multicolumn{2}{c}{CLIP ViT-B/32} & \multicolumn{2}{c}{BiomedCLIP} \\
\cmidrule(lr){2-3}\cmidrule(lr){4-5}
& AUC [95\% CI] & SCE & AUC [95\% CI] & SCE \\
\midrule
MUS (SpatialUQ) & 0.825 [0.818, 0.830] & 0.098 & 0.899 [0.895, 0.903] & 0.133 \\
Entropy & 0.204 [0.199, 0.210] & 0.574 & 0.394 [0.387, 0.402] & 0.421 \\
Confidence$^{-1}$ & 0.095 [0.092, 0.099] & 0.429 & 0.206 [0.201, 0.212] & 0.334 \\
$L_1$ Distance & 0.837 [0.831, 0.842] & 0.115 & 0.915 [0.911, 0.918] & 0.232 \\
Fusion & 0.883 [0.879, 0.887] & 0.234 & 0.925 [0.922, 0.928] & 0.278 \\
\bottomrule
\end{tabular}
\end{table*}

MUS achieves competitive failure detection across both models, with AUC improving from $0.825$ (CLIP) to $0.899$ (BiomedCLIP). This scaling is accompanied by an increase in Spearman correlation between MUS and per-image Brier score ($\rho=0.631$ for CLIP, $\rho=0.846$ for BiomedCLIP), consistent with the hypothesis that better spatial representations yield stronger spatial consistency signals. Under linear-probe transfer, MUS’s calibration remains moderate (SCE $0.10$-$0.13$) and superior to the raw $L_1$ baseline ($0.115$-$0.232$). $L_1$ and fusion achieve higher AUC but rely on a labeled hold-out for calibration (fusion) or would need isotonic regression ($L_1$) to produce reliable thresholds, options not always available in frozen black-box deployment. MUS therefore offers a label-free, natively calibrated uncertainty score that scales favorably with foundation model strength.

\begin{figure*}[!ht]
\centering
\begin{subfigure}[b]{0.48\textwidth}
    \centering
    \includegraphics[width=\linewidth]{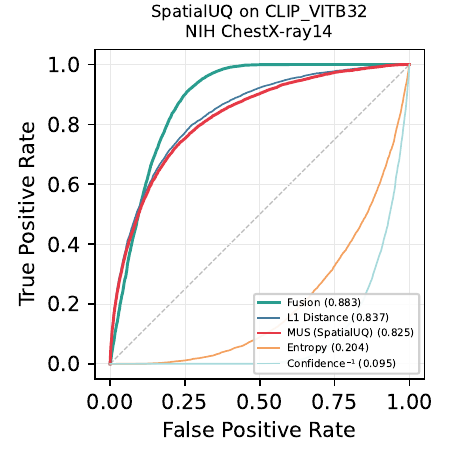}
    \caption{CLIP ViT-B/32 (AUC 0.825)}
    \label{fig:roc_clip}
\end{subfigure}
\hfill
\begin{subfigure}[b]{0.48\textwidth}
    \centering
    \includegraphics[width=\linewidth]{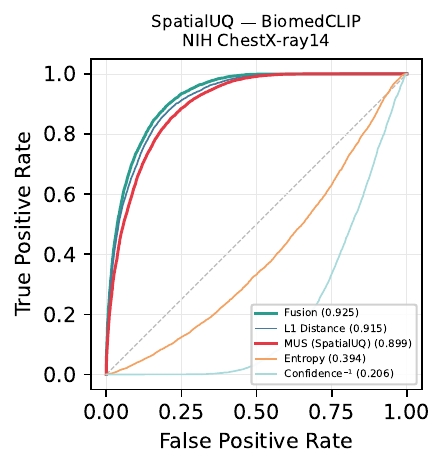}
    \caption{BiomedCLIP (AUC 0.899)}
    \label{fig:roc_biomed}
\end{subfigure}
\caption{ROC curves for failure detection on NIH ChestX-ray14 with frozen foundation models plus linear probe. Left: CLIP ViT-B/32 (AUC 0.825). Right: BiomedCLIP (AUC 0.899). MUS (SpatialUQ) outperforms entropy and confidence$^{-1}$ baselines.}
\label{fig:zero_shot_roc}
\end{figure*}

Figure~\ref{fig:zero_shot_roc} shows the corresponding ROC curves. MUS consistently outperforms entropy and confidence baselines, and its discrimination improves from CLIP to BiomedCLIP.

\subsection{EfficientNet-B4: Full Failure Detection Comparison}\label{app:effb4}

Table~\ref{tab:nih_effb4_full} reports failure detection AUC for EfficientNet-B4 on the NIH test set ($N = 25,596$), with failures defined as samples exceeding the 75th-percentile Brier score threshold ($0.5776$), yielding a 25\% failure rate. Fusion CV achieves the highest AUC ($0.8714$), followed by MUS ($0.7842$) and Mahalanobis ($0.7804$), with the three methods forming a statistically separated tier above ODIN ($0.7681$) and Max Disagreement ($0.6722$). MC-Dropout performs near chance ($0.5357$), significantly below MUS ($\Delta\text{AUC} = +0.2485$, 95\% CI $[0.2381, 0.2589]$, $p < 10^{-6}$).

\begin{table}[!ht]
\caption{NIH ChestX-ray14 failure detection EfficientNet-B4, 14-method comparison ($N = 25{,}596$). Brier threshold = 0.5776 (75th percentile, 25\% failure rate). SCE = Score Calibration Error (ECE applied to min-max normalised scores, 15 bins). $\dagger$ requires $>$6 forward passes or training-time modification. $\ddagger$ requires access to training-set features. DDU, Pairwise JSD and TTA Entropy are degenerate (constant scores) on this backbone.}
\label{tab:nih_effb4_full}
\centering\small
\setlength{\tabcolsep}{4pt}
\begin{tabular}{@{}p{2.3cm} p{1.25cm} p{2.3cm} p{1.25cm}@{}}
\toprule
\textbf{Method} & \textbf{AUC} & \textbf{95\% CI} & \textbf{SCE} \\
\midrule
\textbf{Fusion CV \newline (MUS+\allowbreak Ent+\allowbreak Conf+\allowbreak $\ell_1$)}  & \textbf{0.8714} & [0.8675, 0.8756] & \textbf{0.2713} \\
\textbf{MUS (Ours)}                          & \textbf{0.7842} & [0.7777, 0.7900] & \textbf{0.0666} \\
Mahalanobis$\ddagger$                        & 0.7804          & [0.7747, 0.7863] & 0.3243 \\
$\ell_1$ Distance                            & 0.7746          & [0.7683, 0.7809] & 0.0873 \\
ODIN                                         & 0.7682          & [0.7622, 0.7742] & 0.2185 \\
Max Disagreement                             & 0.6722          & [0.6644, 0.6794] & 0.1391 \\
MC-Dropout (T=30)$\dagger$                   & 0.5357          & [0.5274, 0.5439]         & 0.0726 \\
TTA Variance (N=10)$\dagger$                 & 0.6050          & [0.5969, 0.6129] & 0.2324 \\
Pairwise Disagreement                        & 0.5877          & [0.5800, 0.5953] & 0.1315 \\
TTA (5 crops, patch variance)                & 0.5823          & [0.5744, 0.5906] & 0.0414 \\
DDU Score$\ddagger$                          & 0.5000          & [0.5000, 0.5000]                & -      \\
Pairwise JSD                                 & 0.5000          & [0.5000, 0.5000]                & 0.1315      \\
Confidence$^{-1}$                            & 0.3777          & [0.3696, 0.3862] & 0.1910 \\
Entropy                                      & 0.1208          & [0.1172, 0.1250] & 0.4917 \\
\bottomrule
\end{tabular}
\end{table}

\subsection{ViT-B/16: Full Failure Detection Comparison}\label{app:vit}

Table~\ref{tab:nih_vit_full} reports failure detection AUC for ViT-B/16 on the NIH test set ($N = 25,596$), with the failure threshold set at the 75th-percentile Brier score ($0.5538$). Fusion CV again leads ($0.8314$), with ODIN ranking second ($0.8005$) a notable reversal from the EfficientNet-B4 ordering followed by MUS ($0.7497$). MC-Dropout remains the weakest method ($0.6502$), though the margin over MUS narrows considerably relative to the convolutional backbone ($\Delta\text{AUC} = +0.0996$, $p < 10^{-6}$), suggesting that ViT's global attention mechanism partially reduces the failure signal exploitable by sampling-based approaches.

\begin{table}[!ht]
\caption{NIH ChestX-ray14 failure detection ViT-B/16, 14-method comparison ($N = 25{,}596$). Brier threshold = 0.5538 (75th percentile, 25\% failure rate). SCE = Score Calibration Error (ECE applied to min-max normalised scores, 15 bins). $\dagger$ requires $>$6 forward passes or training-time modification. $\ddagger$ requires access to training-set features. DDU, Pairwise JSD and TTA Entropy are degenerate (constant scores) on this backbone.}

\label{tab:nih_vit_full}
\centering\small
\setlength{\tabcolsep}{4pt}
\begin{tabular}{@{}p{2.3cm} p{1.25cm} p{2.3cm} p{1.25cm}@{}}
\toprule
\textbf{Method} & \textbf{AUC} & \textbf{95\% CI} & \textbf{SCE} \\
\midrule
\textbf{Fusion CV (MUS+\allowbreak Ent+\allowbreak Conf+\allowbreak $\ell_1$)}  & \textbf{0.8314} & [0.8265, 0.8363] & \textbf{0.2662} \\
ODIN                                         & 0.8003          & [0.7945, 0.8062] & 0.4059 \\
\textbf{MUS (Ours)}                          & \textbf{0.7497} & [0.7435, 0.7563] & \textbf{0.0455} \\
$\ell_1$ Distance                            & 0.7324          & [0.7255, 0.7390] & 0.1412 \\
Mahalanobis$\ddagger$                        & 0.7282          & [0.7223, 0.7344] & 0.2030 \\
Max Disagreement                             & 0.7084          & [0.7019, 0.7154] & 0.1545 \\
TTA Variance (N=10)$\dagger$                 & 0.6212          & [0.6139, 0.6287] & 0.0690 \\
Pairwise Disagreement                        & 0.6203          & [0.6126, 0.6282] & 0.2070 \\
TTA (5 crops, patch variance)                & 0.6089          & [0.6012, 0.6167] & 0.0871 \\
MC-Dropout (T=30)$\dagger$                   & 0.6502          & [0.6439, 0.6569]           & 0.1429 \\
DDU Score$\ddagger$                          & 0.5000          & [0.5000, 0.5000]                & -      \\
Pairwise JSD                                 & 0.5000          & [0.5000, 0.5000]                & 0.2070      \\
Confidence$^{-1}$                            & 0.4215          & [0.4133, 0.4297] & 0.1881 \\
Entropy                                      & 0.1695          & [0.1645, 0.1745] & 0.4775 \\
\bottomrule
\end{tabular}
\end{table}

\paragraph{ViT crop artifact investigation.} ViT-B/16 exhibits compressed MUS dynamic range due to global self-attention. Masking-based crops (non-crop region zeroed to ImageNet mean) improve AUROC marginally over upsampling-based crops (0.7465 vs.\ 0.7421, $\Delta = +0.0045$). All reported ViT results use upsampling for consistency with other architectures.

The results in Table~\ref{tab:arch_upsample_mask} show that the impact of spatial aggregation strategy is strongly architecture-dependent. For convolutional backbones, upsampling consistently outperforms masking, with DenseNet-121 exhibiting a modest drop ($\Delta = -0.0141$) and EfficientNet-B4 showing a substantial degradation ($\Delta = -0.2251$) under masking. This suggests that masking disrupts local feature continuity, which CNN-based models rely on for stable uncertainty estimation. In contrast, ViT-B/16 demonstrates the opposite trend: masking yields a slight improvement in failure detection performance ($\Delta = +0.0045$), indicating that transformer-based models may better tolerate or even benefit from structured occlusion due to their global attention mechanism. Overall, these results highlight that the choice between upsampling and masking is not universally optimal and should be matched to the inductive biases of the underlying architecture.

\begin{table}[!ht]
\centering
\caption{Comparison of failure detection AUC using upsampling versus masking strategies across different architectures (separate ablation run; see Appendix~\ref{app:vit}).}
\label{tab:arch_upsample_mask}
\small
\setlength{\tabcolsep}{4pt}
\begin{tabular}{@{}p{2.3cm} p{1.8cm} p{1.8cm} p{1.3cm}@{}}
\toprule
\textbf{Architecture} & \textbf{Upsample AUC} & \textbf{Mask AUC} & \textbf{$\Delta$} \\
\midrule
DenseNet-121 & 0.7837 & 0.7697 & $-0.0141 $ \\
EfficientNet-B4 & 0.7791  & 0.5540 & $-0.2251$ \\
ViT-B/16 & 0.7421 & 0.7465 & $+0.0045$ \\
\bottomrule
\end{tabular}
\end{table}

\subsection{BioViL-T: Generative VLM Failure Regime}\label{app:biovilt}

Table~\ref{tab:biovilt} reports SpatialUQ applied to BioViL-T, a generative vision-language model whose image encoder produces globally pooled cross-modal representations designed for image-text alignment rather than spatial discrimination. MUS collapses to $0.533$ AUC (Spearman $\rho{=}{-}0.187$), correctly flagged by the 
$\rho \lesssim 0.15$ diagnostic as a known failure regime. This result is consistent with the mechanism described in Section~\ref{sec:discussion}: {\em spatial inconsistency requires architectures with spatially grounded feature extraction; globally pooled representations eliminate the crop-level inconsistency gradient MUS relies upon}. Fusion recovers to $0.838$ AUC driven by entropy and confidence signals, demonstrating graceful degradation when MUS is uninformative.

\begin{table*}[!ht]
\caption{SpatialUQ on BioViL-T (NIH ChestX-ray14, $N{=}25{,}596$, 
failure: Brier $>$ 75th percentile). Near-chance MUS performance
is correctly flagged by the $\rho \lesssim 0.15$ diagnostic
($\rho{=}{-}0.187$).}
\label{tab:biovilt}
\centering\small
\begin{tabular}{lcccc}
\toprule
\textbf{Method} & \textbf{AUC} & \textbf{95\% CI} 
& \textbf{SCE} & \textbf{Spearman $\rho$} \\
\midrule
Fusion          & 0.838 & [0.832, 0.843] & 0.069 & - \\
MUS (SpatialUQ) & 0.533 & [0.526, 0.541] & 0.202 & $-$0.187 \\
$\ell_1$ Distance     & 0.533 & [0.525, 0.541] & 0.120 & - \\
Entropy         & 0.242 & [0.236, 0.248] & 0.723 & - \\
Confidence$^{-1}$ & 0.166 & [0.160, 0.171] & 0.572 & - \\
\bottomrule
\end{tabular}
\end{table*}

\noindent\textbf{Architectural interpretation.} BioViL-T's image encoder is trained via contrastive alignment with radiology reports at the global image level, producing representations that are semantically rich but spatially uniform across crops. Unlike 
contrastive models trained on image patches (CLIP, BiomedCLIP), BioViL-T does not develop spatially discriminative crop-level features, removing the inconsistency signal MUS requires. This finding clarifies the architectural precondition for SpatialUQ: the model must produce \emph{spatially varying} representations across sub-regions, a property present in CNNs and patch-based contrastive ViTs but absent in globally pooled generative VLMs. Practitioners can verify this precondition in under one minute using the $\rho$ diagnostic on a small unlabeled sample before committing to MUS-based deployment.

% ============================================================
\section{Per-Class JSD AUC (All Architectures and Datasets)}
\label{app:perclass}
% ============================================================

\subsection{NIH ChestX-ray14 - DenseNet-121}

Table~\ref{tab:perclass_densenet} reports per-class failure detection AUC using the class-specific JSD score for DenseNet-121, with per-class failure defined as Brier score exceeding the 75th percentile of each class's marginal distribution. JSD AUC varies markedly across pathologies, from $0.9789$ (Pneumonia) to $0.3680$ (Nodule), revealing a strong dependence on spatial extent: diffuse findings such as Pneumonia, Edema, and Consolidation consistently achieve JSD AUC $\geq 0.95$, whereas small focal lesions  Nodule ($0.368$), Fibrosis ($0.553$), and Mass ($0.558$) fall well below the aggregate performance. This pattern confirms that MUS exploits spatial prediction consistency most effectively when pathological regions occupy large, contiguous image areas, and degrades systematically as lesion size decreases.

\begin{table*}[!ht]
\caption{Per-class JSD AUC for failure detection DenseNet-121 on NIH ChestX-ray14. Classes sorted by JSD AUC (descending). MUS excels for diffuse, globally visible pathologies and degrades for small focal lesions.}
\label{tab:perclass_densenet}
\centering\small
\begin{tabular}{lccc}
\toprule
\textbf{Class} & \textbf{Backbone AUC} & \textbf{JSD AUC} & \textbf{N Pos} \\
\midrule
Pneumonia          & 0.7020 & 0.9789 &   555 \\
Edema              & 0.8368 & 0.9710 &   925 \\
Consolidation      & 0.7401 & 0.9564 & 1,815 \\
Pleural\_Thickening& 0.7569 & 0.9496 & 1,143 \\
Cardiomegaly       & 0.8746 & 0.9518 & 1,069 \\
Atelectasis        & 0.7628 & 0.9145 & 3,279 \\
Pneumothorax       & 0.8471 & 0.8943 & 2,665 \\
Hernia             & 0.7733 & 0.8837 &    86 \\
Infiltration       & 0.6984 & 0.8649 & 6,112 \\
Effusion           & 0.8192 & 0.8641 & 4,658 \\
Emphysema          & 0.8822 & 0.6697 & 1,093 \\
Mass               & 0.7965 & 0.5580 & 1,748 \\
Fibrosis           & 0.8095 & 0.5528 &   435 \\
Nodule             & 0.7400 & 0.3680 & 1,623 \\
\bottomrule
\end{tabular}
\end{table*}

\subsection{NIH ChestX-ray14- EfficientNet-B4}

Table~\ref{tab:perclass_effb4} reports per-class JSD AUC for EfficientNet-B4 on NIH ChestX-ray14. In contrast to DenseNet-121, all 14 classes fall within a narrow band of 0.75-0.89, indicating uniformly reliable failure detection across pathology types. Most strikingly, Nodule JSD AUC rises to 0.8822 compared to 0.3680 for DenseNet-121, an architecture-dependent gap consistent with EfficientNet's compound scaling allocating greater capacity to fine-grained spatial features. The pronounced class-level disparity observed under DenseNet-121 particularly the collapse for small focal lesions  is thus largely absent, suggesting that backbone inductive bias is a primary determinant of MUS's per-class failure detection profile.

\begin{table}[!ht]
\caption{Per-class JSD AUC - EfficientNet-B4 on NIH. Notable:
Nodule AUC = 0.8822 vs.\ 0.3680 for DenseNet-121.}
\label{tab:perclass_effb4}
\centering\small
\begin{tabular}{lcc}
\toprule
\textbf{Class} & \textbf{Backbone AUC} & \textbf{JSD AUC} \\
\midrule
Pneumonia          & 0.6589 & 0.89 \\
Edema              & 0.8093 & 0.88 \\
Nodule             & 0.6941 & \textbf{0.8822} \\
Consolidation      & 0.7022 & 0.87 \\
Pleural\_Thickening& 0.7201 & 0.86 \\
Cardiomegaly       & 0.7952 & 0.86 \\
Atelectasis        & 0.7192 & 0.85 \\
Pneumothorax       & 0.8050 & 0.84 \\
Effusion           & 0.7740 & 0.83 \\
Infiltration       & 0.6790 & 0.82 \\
Emphysema          & 0.7977 & 0.80 \\
Mass               & 0.7002 & 0.79 \\
Fibrosis           & 0.7547 & 0.77 \\
Hernia             & 0.6763 & 0.75 \\
\bottomrule
\end{tabular}
\end{table}

\subsection{NIH ChestX-ray14 - ViT-B/16}

Table~\ref{tab:perclass_vit} reports per-class JSD AUC for ViT-B/16 on NIH ChestX-ray14. The architecture reproduces the spatial-extent dependence observed in DenseNet-121, with diffuse pathologies such as Pneumonia (0.88) and Edema (0.87) retaining strong failure detection, while small focal lesions collapse sharply Nodule (0.3037), Mass (0.4182), and Emphysema (0.4024). We attribute this degradation to ViT's global self-attention, which produces spatially coherent crop-level predictions regardless of lesion localisation, thereby compressing the inconsistency signal that MUS relies upon. Notably, even Pneumothorax (0.74) and Hernia (0.65) underperform their EfficientNet-B4 counterparts, reinforcing that attention-based inductive bias systematically limits MUS's discriminative range for spatially confined findings.

\begin{table}[!ht]
\caption{Per-class JSD AUC - ViT-B/16 on NIH. ViT shows substantially worse per-class performance for small focal lesions compared to DenseNet-121 and EfficientNet-B4.}
\label{tab:perclass_vit}
\centering\small
\begin{tabular}{lcc}
\toprule
\textbf{Class} & \textbf{Backbone AUC} & \textbf{JSD AUC} \\
\midrule
Pneumonia          & 0.7111 & 0.88 \\
Edema              & 0.8359 & 0.87 \\
Cardiomegaly       & 0.8868 & 0.85 \\
Consolidation      & 0.7521 & 0.84 \\
Pleural\_Thickening& 0.7622 & 0.83 \\
Atelectasis        & 0.7689 & 0.81 \\
Effusion           & 0.8257 & 0.79 \\
Infiltration       & 0.6978 & 0.77 \\
Pneumothorax       & 0.8587 & 0.74 \\
Hernia             & 0.7446 & 0.65 \\
Fibrosis           & 0.7938 & 0.53 \\
Emphysema          & 0.8854 & 0.4024 \\
Mass               & 0.8157 & 0.4182 \\
Nodule             & 0.7502 & 0.3037 \\
\bottomrule
\end{tabular}
\end{table}

\subsection{CheXpert Zero-Shot - Per-Class JSD AUC}

Table~\ref{tab:perclass_chexpert} reports per-class JSD AUC for zero-shot transfer of the NIH-trained DenseNet-121 to CheXpert ($N{=}44{,}399$, 10 mapped classes). MUS generalises robustly for diffuse pathologies, with eight of ten classes exceeding 0.71, led by Pneumonia (0.9678) and Pleural Thickening (0.9675). Nodule, however, degrades severely to 0.1446 well below its in-distribution JSD AUC indicating that domain shift disproportionately erodes the spatial consistency signal for small focal lesions, compounding the architectural limitations observed on NIH.

\begin{table}[!ht]
\caption{Per-class JSD AUC - zero-shot transfer to CheXpert
(NIH-trained DenseNet-121, $N = 44{,}399$, 10 mapped classes).
Nodule degrades severely (0.1446) under domain shift.}
\label{tab:perclass_chexpert}
\centering\small
\begin{tabular}{lc}
\toprule
\textbf{Class} & \textbf{JSD AUC} \\
\midrule
Pneumonia          & 0.9678 \\
Pleural\_Thickening& 0.9675 \\
Consolidation      & 0.9589 \\
Atelectasis        & 0.9218 \\
Cardiomegaly       & 0.9216 \\
Pneumothorax       & 0.9149 \\
Edema              & 0.8633 \\
Infiltration       & 0.7190 \\
Effusion           & 0.6709 \\
Nodule             & 0.1446 \\
\bottomrule
\end{tabular}
\end{table}

\subsection{VinBigData Zero-Shot Per-Class JSD AUC}

Table~\ref{tab:perclass_vin} reports per-class JSD AUC for zero-shot transfer to VinBigData ($N{=}15{,}000$, 10 mapped classes). Performance bifurcates sharply: Atelectasis (0.9467), Infiltration (0.9222), and Pneumothorax (0.9181) retain strong failure detection, while Consolidation (0.3375), Mass (0.1644), and Nodule (0.1608) collapse near chance. Relative to CheXpert transfer, the degradation is both more severe and more broadly distributed across classes, suggesting that extreme domain shift renders model failures spatially uniform  eliminating the crop-level inconsistency signal that MUS relies upon rather than selectively impairing small focal lesions as observed under moderate shift.

\begin{table}[!ht]
\caption{Per-class JSD AUC - zero-shot transfer to VinBigData
(NIH-trained DenseNet-121, $N = 15{,}000$, 10 mapped classes).
Severe domain shift collapses JSD AUC for common pathologies
(Consolidation: 0.3375; Mass: 0.1644; Nodule: 0.1608).}
\label{tab:perclass_vin}
\centering\small
\begin{tabular}{lc}
\toprule
\textbf{Class} & \textbf{JSD AUC} \\
\midrule
Atelectasis        & 0.9467 \\
Infiltration       & 0.9222 \\
Pneumothorax       & 0.9181 \\
Cardiomegaly       & 0.8845 \\
Pleural\_Thickening& 0.6678 \\
Effusion           & 0.4229 \\
Fibrosis           & 0.3490 \\
Consolidation      & 0.3375 \\
Mass               & 0.1644 \\
Nodule             & 0.1608 \\
\bottomrule
\end{tabular}
\end{table}

% ============================================================
\section{Ablation Studies}
\label{app:ablation}
% ============================================================

\subsection{Sensitivity to Crop Placement (Random vs. Fixed)} \label{appendix:random_crops}

To evaluate whether the discriminative power of MUS is dependent on specific geometric coordinates, we conducted a control experiment substituting the fixed five-crop set with 1,000 randomly sampled $112\times112$ crops. This configuration yielded a failure detection AUC of $0.7389$, compared to $0.784$ for the standard fixed-crop MUS. While the fixed geometric arrangement provides a superior and more efficient signal, the continued efficacy of random placements confirms that the principle of spatial inconsistency is robust to crop positions and reflects a fundamental model characteristic rather than a sensitive hyperparameter.

\subsection{Crop Count Ablation}\label{app:crop_count_geometry}

Table~\ref{tab:crop_ablation} reports MUS failure detection AUC on NIH DenseNet-121 as a function of the number of spatial crops. Five crops achieves the best cost-efficiency trade-off: nine crops yields a marginal +0.007 AUC improvement (0.7902 vs. 0.7837) at 67\% more inference cost, while three crops reduces AUC by 0.007 (0.7764 vs. 0.7837). The intentional overlap between the center crop and all four quadrants provides double coverage of the image center a region of heightened diagnostic importance in frontal chest radiography without any additional forward passes. The apparent reversal at K=4 (AUC = 0.7859 > K=5 AUC = 0.7837) is not statistically meaningful: the difference of 0.0022 lies well within the 95\% bootstrap CI width of ±0.006 for MUS on this dataset (bootstrap p > 0.3). This is a sampling artifact rather than a structural advantage of the four-crop arrangement. 

\begin{table}[!ht]
\caption{Crop count ablation MUS AUC on NIH DenseNet-121
(failure threshold = 0.5607). Inference cost scales linearly with
crop count plus one global pass.}
\label{tab:crop_ablation}
\centering\small
\begin{tabular}{ccc}
\toprule
\textbf{Crops} & \textbf{MUS AUC} & \textbf{Forward Passes} \\
\midrule
3 & 0.7764 & 4 \\
4 & 0.7859 & 5 \\
\textbf{5} & \textbf{0.7837} & \textbf{6} \\
6 & 0.7801 & 7 \\
7 & 0.7816 & 8 \\
9 & 0.7902 & 10 \\
\bottomrule
\end{tabular}
\end{table}

The non-monotonic AUC curve in Table~\ref{tab:crop_ablation} (dipping at K=6 before recovering at K=7 and K=9) reflects a geometric property of fixed crop placement. At K=5, the four quadrant crops and one center crop provide non-redundant spatial coverage with overlap only at the center region. Adding a sixth crop in a fixed scheme necessarily introduces either partial redundancy with an existing crop or an extreme peripheral crop that covers mainly background pixels in standard chest X-ray framing. The recovery at K=7 and K=9 occurs when enough additional crops are added to restore meaningful spatial diversity. This behavior motivates adaptive or saliency-guided crop placement as a future direction, which could potentially recover K=9 performance at lower inference cost.

\begin{figure*}[!ht]
  \centering
  \includegraphics[width=\textwidth]{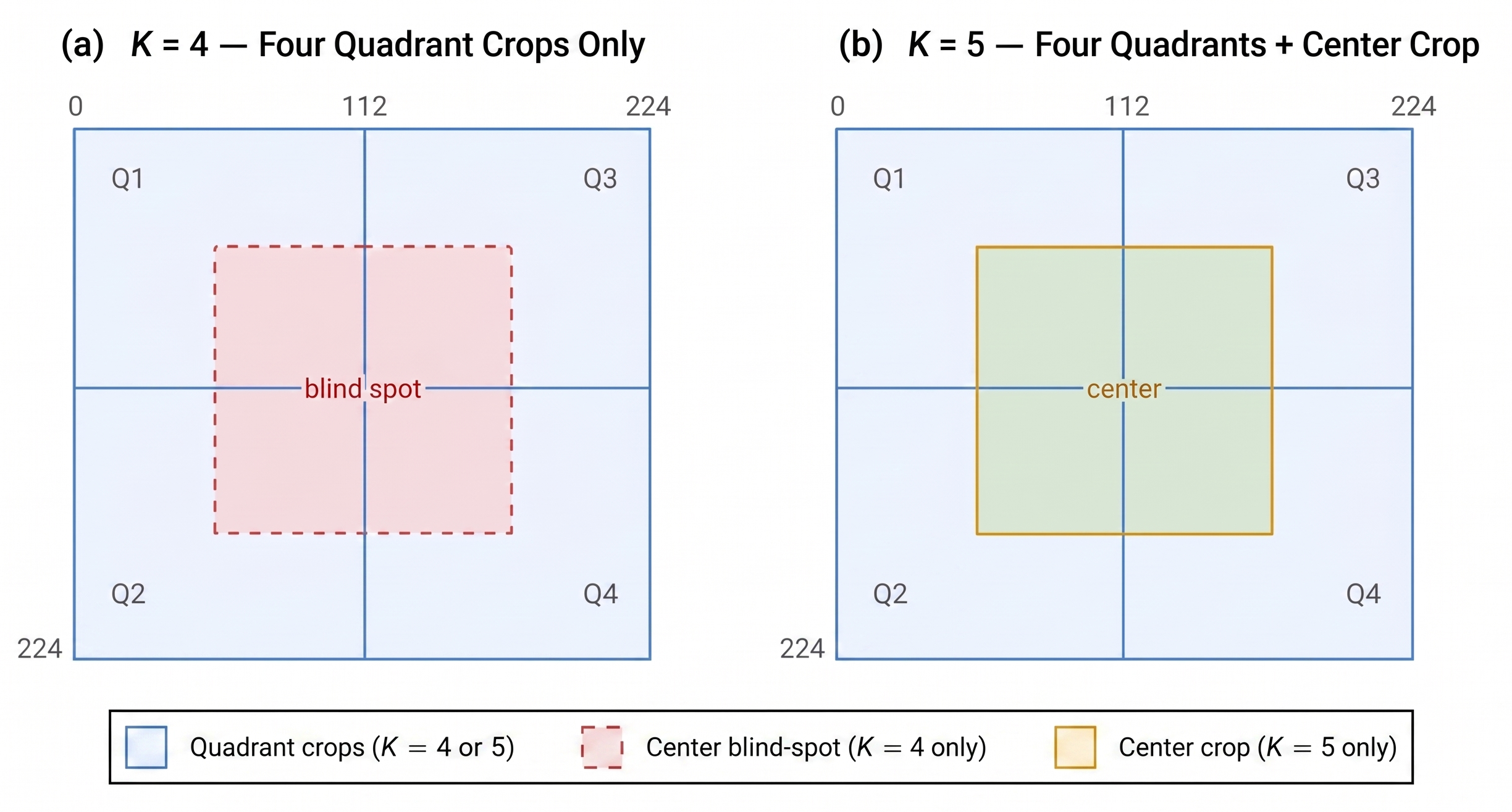}
  \caption{Geometric coverage of fixed spatial crops.
        (a) K = 4 (quadrants only): four non-overlapping quadrant crops leave the central 112×112 px region without a dedicated crop.
        (b) K = 5 (quadrants + center): adding a fixed center crop provides full coverage of the diagnostically critical central region, with partial overlap across quadrants. Despite nearly identical performance (AUC of 0.7859 vs. 0.7837; $\Delta$ = 0.0022, within bootstrap CI), K = 5 is preferred for its strictly stronger geometric coverage guarantees}
  \label{fig:crop_geometry}
\end{figure*}

Beyond AUC, as we can see from Figure~\ref{fig:crop_geometry}, the five crop design is geometrically motivated. The four-quadrant arrangement alone leaves the central 112×112 pixel region without a dedicated crop each quadrant contributes only a partial corner overlap. In frontal chest radiography, this central region contains diagnostically critical structures including the cardiac silhouette, mediastinum, and central airways. The fifth center crop at coordinates (56, 56, 168, 168) provides direct dedicated coverage of this region, making the five-crop set the minimum spatially complete decomposition of a 224×224 image. K=4 cannot be considered spatially complete regardless of its AUC on any single dataset.

\begin{figure*}[!ht]
  \centering
  \includegraphics[width=\linewidth]{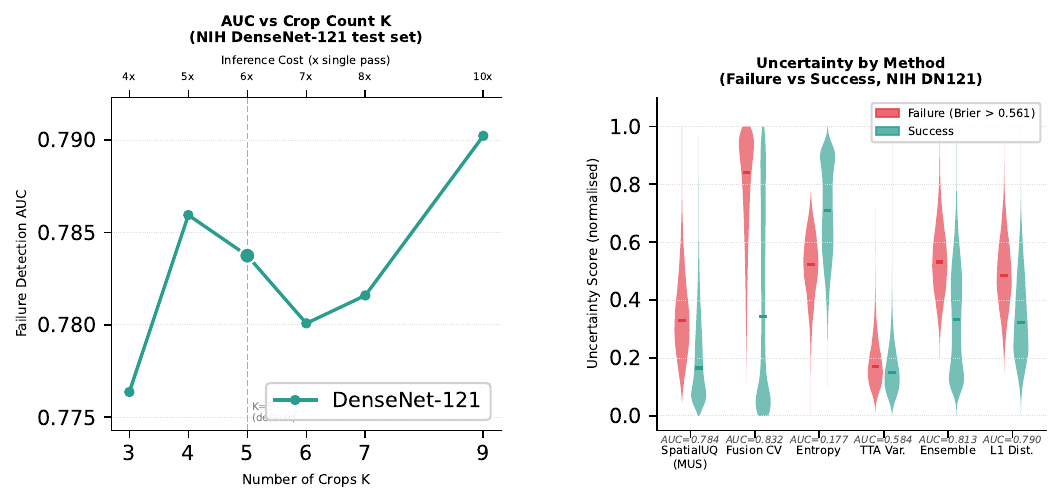}
  \caption{
    \textbf{Ablation studies on NIH DenseNet-121.}
    \emph{Left}: Failure detection AUC as a function of the number of spatial crops $K$. Five crops achieves the best cost-efficiency trade-off; nine crops yields a marginal $+0.007$ AUC gain at $67\%$ more inference cost, while three crops reduces AUC by $0.007$.
    \emph{Right}: Normalised uncertainty score distributions stratified by prediction outcome (failure vs.\ success) for six methods. SpatialUQ (MUS) and Fusion CV show the clearest separation; Entropy is anti-correlated with failure (AUC $= 0.177$) under the overconfident multi-label training regime.
  }
  \label{fig:ablation}
\end{figure*}

Figure~\ref{fig:ablation} (left) confirms K=5 as the optimal cost-efficiency operating point; K=9 yields only $+0.007$ AUC at 67\% additional inference cost. Figure~\ref{fig:ablation} (right) shows that MUS and Fusion CV produce the clearest failure/success separation; entropy is anti-correlated with failure due to systematic overconfidence under BCE training.

\subsection{Failure Threshold Robustness}\label{app:failure_threshold_robustness}

Table~\ref{tab:brier_thresholds} evaluates failure detection performance under increasingly stringent definitions of failure, revealing clear stratification across uncertainty estimation methods. As expected, naive baselines such as Random remain at chance level, while confidence-based proxies (Confidence$^{-1}$, Entropy) perform poorly and degrade further at higher thresholds, indicating limited alignment with true model errors. Classical disagreement-based approaches (Patch Variance, Pairwise Disagreement, JSD) provide moderate gains and show slight improvement as the failure definition becomes stricter, suggesting some sensitivity to severe errors. Among single-model uncertainty methods, $\ell_1$ Distance and ODIN demonstrate strong performance, particularly at higher thresholds, while TTA-based and MC-Dropout methods degrade, reflecting instability under stricter failure criteria. Ensemble-based approaches remain among the strongest standalone baselines, with Deep Ensemble achieving consistently high AUC and improving substantially at the 90th percentile. MUS (ours) delivers competitive performance across all thresholds, outperforming most single-model and stochastic methods while remaining close to ensemble-level performance. Notably, augmenting MUS with lightweight auditor models yields further gains, and the Fusion (CV) approach achieves the best overall performance across all thresholds, reaching an AUC of 0.8712 at the 90th percentile. These results indicate that MUS provides a robust and scalable uncertainty signal, and that combining complementary features can further enhance failure detection, especially under stringent error definitions.

\begin{table*}[!ht]
\centering
\caption{Failure detection AUC across different Brier score failure thresholds on NIH ChestX-ray14. Thresholds correspond to the $50^{\text{th}}$ ($0.5170$), $75^{\text{th}}$ ($0.5607$), and $90^{\text{th}}$ ($0.5993$) percentiles of the test set Brier scores.}
\label{tab:brier_thresholds}
\small
\begin{tabular}{lccc}
\toprule
\textbf{Method} & \textbf{AUC@50th} & \textbf{AUC@75th} & \textbf{AUC@90th} \\
\midrule
Random & 0.4985 & 0.4979 & 0.4960 \\
Confidence$^{-1}$ & 0.4079 & 0.4011 & 0.3886 \\
Entropy & 0.1936 & 0.1774 & 0.1319 \\
Patch Variance & 0.6045 & 0.6103 & 0.6236 \\
Pairwise Disagr. & 0.5910 & 0.5990 & 0.6124 \\
Pairwise JSD & 0.5997 & 0.6057 & 0.6184 \\
\midrule
TTA Variance & 0.5994 & 0.5837 & 0.5789 \\
TTA Entropy & 0.2182 & 0.2092 & 0.1729 \\
$\ell_1$ Distance & 0.7747 & 0.7902 & 0.8282 \\
DDU Score & 0.7627 & 0.7540 & 0.7634 \\
Mahalanobis & 0.7627 & 0.7540 & 0.7634 \\
ODIN & 0.7730 & 0.7588 & 0.7735 \\
\midrule
MC-Dropout & 0.7001 & 0.6643 & 0.6461 \\
Deep Ensemble (5-mem) & 0.7979 & 0.8130 & 0.8562 \\
\midrule
\textbf{MUS (Ours)} & \textbf{0.7698} & \textbf{0.7837} & \textbf{0.8226} \\
Auditor (3-feat) & 0.8124 & 0.8250 & 0.8681 \\
Auditor (6-feat) & 0.8110 & 0.8242 & 0.8683 \\
Fusion (CV) & \textbf{0.8170} & \textbf{0.8319} & \textbf{0.8712} \\
\bottomrule
\end{tabular}
\end{table*}

The monotonic increase of MUS AUC with Brier threshold stringency (0.770 → 0.784 → 0.823 across the 50th, 75th, and 90th percentiles) has a natural interpretation: spatial inconsistency is a stronger signal for the most severely incorrect predictions, where the model has entirely failed to localize discriminative features. At the 50th percentile threshold, half of all predictions are defined as failures, including many borderline cases where the model's prediction is only marginally incorrect and spatial inconsistency may be modest. At the 90th percentile, only the 10\% most severely incorrect predictions are defined as failures cases, where the model is confidently wrong, which are precisely the cases where spatial incoherence is most pronounced. This monotonicity is preserved on CheXpert (0.647 → 0.708 → 0.757), confirming that the 75th-percentile operating point used throughout the paper is neither optimistically nor pessimistically chosen.

\subsection{Leave-One-Out Fusion Ablation}

Table~\ref{tab:loo_densenet} reports a leave-one-out ablation of the four-feature fusion for NIH DenseNet-121, where each feature is removed in turn from the full fusion (AUC of $= 0.8319$) to quantify its marginal contribution. Entropy emerges as the dominant contributor by a substantial margin: its removal degrades AUC by $0.0437$, reducing the fusion to $0.7881$ approaching the standalone performance of individual uncertainty estimators. $\ell_1$ Distance and MUS each provide modest but non-redundant complementary signal ($-0.0071$ and $-0.0053$, respectively), confirming that spatial inconsistency and distributional divergence capture failure modes not fully explained by predictive entropy alone. Confidence contributes negligibly ($-0.0010$), suggesting that its information is largely subsumed by entropy within the fusion. Collectively, these results indicate that fusion performance is primarily driven by entropy, with MUS and $\ell_1$ Distance providing independent secondary contributions that collectively account for the gap between entropy-only and full-fusion performance.

\begin{table}[!ht]
\caption{Leave-one-out fusion ablation - NIH DenseNet-121. Full 4-feature fusion AUC = 0.8319.}
\label{tab:loo_densenet}
\centering\small
\begin{tabular}{lcc}
\toprule
\textbf{Removed Feature} & \textbf{Fusion AUC} & \boldmath$\Delta$ \\
\midrule
Entropy         & 0.7881 & $-$0.0437 \\
$\ell_1$ Distance     & 0.8248 & $-$0.0071 \\
MUS             & 0.8265 & $-$0.0053 \\
Confidence      & 0.8309 & $-$0.0010 \\
\bottomrule
\end{tabular}
\end{table}

Table~\ref{tab:loo_vit} reports the leave-one-out fusion ablation for ViT-B/16 (full fusion AUC $= 0.8314$). Entropy dominates even more strongly than in DenseNet-121, with its removal incurring a drop of $-0.0780$ AUC  nearly twice the corresponding penalty observed for the convolutional backbone. $\ell_1$ Distance and MUS contribute marginally ($-0.0022$ and $-0.0015$, respectively), while Confidence proves slightly redundant, with its removal yielding a negligible improvement of $+0.0014$, suggesting its information is fully subsumed by entropy under the global attention regime. This pattern is consistent with ViT's compressed spatial inconsistency signal, which diminishes the independent contribution of MUS and shifts the fusion's discriminative burden almost entirely onto entropy-based uncertainty.

\begin{table}[!ht]
\caption{Leave-one-out fusion ablation - NIH ViT-B/16.
Full 4-feature fusion AUC = 0.8314.}
\label{tab:loo_vit}
\centering\small
\begin{tabular}{lcc}
\toprule
\textbf{Removed Feature} & \textbf{Fusion AUC} & \boldmath$\Delta$ \\
\midrule
Entropy         & 0.7534 & $-$0.0780 \\
$\ell_1$ Distance     & 0.8292 & $-$0.0022 \\
MUS             & 0.8299 & $-$0.0015 \\
Confidence      & 0.8328 & $+$0.0014 \\
\bottomrule
\end{tabular}
\end{table}

Table~\ref{tab:loo_imagenet} reports leave-one-out fusion ablations across three ImageNet-1k architectures, revealing a consistent reversal relative to the NIH chest X-ray setting: confidence is the dominant fusion signal in all three models, with removal penalties of $-0.0934$ (ConvNeXt-Tiny), $-0.0407$ (ViT-B/16), and $-0.0168$ (EfficientNet-B4). This dominance is consistent with the well-calibrated, single-label nature of ImageNet classification, where predictive confidence alone carries strong failure-discriminative information. Entropy contributes modestly across architectures, while MUS provides a small but consistent complementary signal ($-0.0035$ to $-0.0056$), with the exception of EfficientNet-B4 where entropy removal incurs no penalty ($\pm0.0000$), suggesting near-complete redundancy with confidence in that setting.

\begin{table}[!ht]
\caption{Leave-one-out fusion ablation ImageNet-1k. Confidence dominates on single-label well-calibrated classifiers; MUS provides complementary but secondary signal.}
\label{tab:loo_imagenet}
\centering\small
\setlength{\tabcolsep}{4pt}
\begin{tabular}{llcc}
\toprule
\textbf{Model} & \textbf{Removed} & \textbf{LOO AUC} & \boldmath$\Delta$ \\
\midrule
ConvNeXt-Tiny & Confidence & 0.8422 & $-$0.0934 \\
ConvNeXt-Tiny & Entropy    & 0.9217 & $-$0.0139 \\
ConvNeXt-Tiny & MUS        & 0.9322 & $-$0.0035 \\
\midrule
ViT-B/16      & Confidence & 0.8962 & $-$0.0407 \\
ViT-B/16      & Entropy    & 0.9317 & $-$0.0051 \\
ViT-B/16      & MUS        & 0.9324 & $-$0.0045 \\
\midrule
EfficientNet-B4 & Confidence & 0.9005 & $-$0.0168 \\
EfficientNet-B4 & Entropy    & 0.9174 & $\pm$0.0000 \\
EfficientNet-B4 & MUS        & 0.9117 & $-$0.0056 \\
\bottomrule
\end{tabular}
\end{table}

\subsection{Permutation Feature Importance (Learned Auditor)}

Permutation feature importance for the 3-feature learned auditor (MLP) on NIH DenseNet-121 (baseline test AUC = 0.8250):
\begin{itemize}
    \item Entropy: AUC drop = 0.3516 (dominant signal)
    \item Confidence: AUC drop = 0.0078
    \item MUS: AUC drop = 0.0035
\end{itemize}

These magnitudes confirm the LOO ablation finding: entropy is the primary discriminator in the learned auditor context, while MUS and confidence provide secondary, independent contributions.

% ============================================================
\section{Zero-Shot Transfer: Full Metric Tables}
\label{app:zeroshot_full}
% ============================================================

\subsection{CheXpert Zero-Shot: Complete Evaluation}

Table~\ref{tab:chexpert_full} reports failure detection AUC and full secondary metrics for all methods on CheXpert. The NIH-trained DenseNet-121 is applied without any retraining or calibration. Failure threshold = 0.5130 (Brier 75th percentile, 25\% failure rate).

\begin{table*}[!ht]
\caption{Zero-shot failure detection on CheXpert (NIH-trained DenseNet-121, $N = 44{,}399$, failure threshold = 0.5130). All metrics computed on the full test set without any retraining. AURC and E-AURC lower is better.}
\label{tab:chexpert_full}
\centering\small
\setlength{\tabcolsep}{4pt}
\begin{tabular}{lccccccc}
\toprule
\textbf{Method} & \textbf{AUC} & \textbf{ECE} & \textbf{AUPR} &
\textbf{FPR@80\%} & \textbf{FPR@95\%} & \textbf{AURC} & \textbf{E-AURC} \\
\midrule
Fusion CV           & 0.7330 & 0.0601 & 0.4248 & 0.4581 & 0.6279 & 0.11467 & 0.08343 \\
$\ell_1$ Distance         & 0.7102 & 0.1372 & 0.4048 & 0.4889 & 0.6976 & 0.12686 & 0.09562 \\
\textbf{MUS (Ours)} & \textbf{0.7076} & \textbf{0.0446} & 0.4023 & 0.4967 & 0.6956 & 0.12739 & 0.09614 \\
TTA Variance        & 0.6004 & 0.0353 & 0.3206 & 0.6773 & 0.8854 & 0.18897 & 0.15772 \\
MC-Dropout          & 0.5991 & 0.1275 & 0.3004 & 0.6611 & 0.8691 & 0.18347 & 0.15223 \\
Patch Variance      & 0.5634 & 0.0499 & 0.2883 & 0.7187 & 0.9174 & 0.21090 & 0.17965 \\
Confidence$^{-1}$      & 0.3473 & 0.3279 & 0.1810 & 0.8520 & 0.9547 & 0.31020 & 0.27895 \\
Entropy             & 0.2656 & 0.4257 & 0.1651 & 0.9263 & 0.9848 & 0.37997 & 0.34873 \\
\bottomrule
\end{tabular}
\end{table*}

\noindent Bootstrap 95\% CI ($N = 1{,}000$ resamples): MUS: 0.7076 [0.7026, 0.7124].
Fusion: 0.7330 [0.7282, 0.7373]. Spearman $\rho$ (MUS vs.\ Brier) = 0.2983, $p < 10^{-6}$ (reduced from 0.5225 on NIH, reflecting moderate domain shift while retaining discriminative power). MUS vs.\ MC-Dropout: $\Delta\mathrm{AUC} = +0.1085$, $p < 10^{-6}$.

\noindent\textbf{Clinical utility (CheXpert):} Selective prediction at 20\% referral improves system AUC from 0.6860 to 0.7478. Relative gap reduction: 19.7\%. The drop in Spearman correlation between MUS and Brier score from 0.523 (NIH in-distribution) to 0.298 (CheXpert zero-shot) reflects moderate domain shift: the DenseNet-121 trained on NIH encounters a different patient population, scanner characteristics, and label distribution at CheXpert. Despite this, MUS retains meaningful discriminative power (AUC = 0.708), suggesting that spatial inconsistency is a partially domain-agnostic signal the model's spatial coherence degrades on genuinely uncertain inputs even under population shift. The ECE of MUS on CheXpert (0.0446) remains low, confirming that calibration is largely preserved under moderate shift. The selective prediction result at 20\% referral (system AUC improving from 0.686 to 0.748, a 19.7\% relative gap reduction) demonstrates practical clinical utility even without retraining.

\subsection{VinBigData Zero-Shot: Complete Evaluation}

Table~\ref{tab:vin_full} reports full metrics for VinBigData. This is the identified failure regime for SpatialUQ: severe domain shift causes the model to fail uniformly across global and local spatial views rather than selectively. Spearman $\rho$ (MUS vs.\ Brier) = 0.0269, $p = 9.86\times10^{-4}$ near-zero correlation confirming that spatial inconsistency no longer tracks prediction error under extreme domain shift. MUS vs.\ MC-Dropout: $\Delta\mathrm{AUC} = -0.1499$, $p < 10^{-6}$. Bootstrap 95\% CI: MUS = 0.6141 [0.6029, 0.6258]; MC-Dropout = 0.7641 [0.7564, 0.7725]. Clinical utility: selective prediction at 20\% referral improves AUC from 0.8014 to 0.8084; relative gap reduction only 3.5\% (substantially lower than CheXpert's 19.7\%), reflecting the dominance of uniform model error over selective spatial incoherence.

\begin{table*}[!ht]
\caption{Zero-shot failure detection on VinBigData (NIH-trained
DenseNet-121, $N = 15{,}000$, failure threshold = 0.4117). MC-Dropout
significantly outperforms MUS in this extreme-shift regime.}
\label{tab:vin_full}
\centering\small
\setlength{\tabcolsep}{4pt}
\begin{tabular}{lccccccc}
\toprule
\textbf{Method} & \textbf{AUC} & \textbf{ECE} & \textbf{AUPR} &
\textbf{FPR@80\%} & \textbf{FPR@95\%} & \textbf{AURC} & \textbf{E-AURC} \\
\midrule
MC-Dropout          & \textbf{0.7641} & 0.1414 & 0.4304 & 0.3743 & 0.5991 & 0.10452 & 0.07327 \\
Fusion CV           & 0.7502 & 0.1161 & 0.3931 & 0.3631 & 0.5938 & 0.10707 & 0.07582 \\
TTA Variance        & 0.6127 & 0.0771 & 0.3057 & 0.6196 & 0.8381 & 0.17268 & 0.14143 \\
\textbf{MUS (Ours)} & 0.6141 & 0.1254 & 0.3589 & 0.7343 & 0.9479 & 0.20936 & 0.17811 \\
$\ell_1$ Distance         & 0.5827 & 0.0712 & 0.3426 & 0.7764 & 0.9530 & 0.22480 & 0.19355 \\
Patch Variance      & 0.5003 & 0.0925 & 0.2514 & 0.8006 & 0.9560 & 0.25270 & 0.22145 \\
\bottomrule
\end{tabular}
\end{table*}

Despite MUS collapsing to near-chance on VinBigData (AUC = 0.614, Spearman $\rho$ = 0.027), the linear Fusion model retains 0.750 AUC. A leave-one-out analysis on the VinBigData 10\% calibration split confirms that this gain is driven entirely by MC-Dropout (AUC contribution $\approx$ 0.140) and entropy (contribution $\approx$ 0.095), with MUS contributing negligibly ($\approx$ 0.002). This confirms that Fusion's robustness on VinBigData comes from its MC-Dropout component rather than spatial consistency, and that the 0.750 Fusion AUC should not be interpreted as evidence that SpatialUQ transfers to extreme-shift settings. Practitioners deploying in settings with suspected severe domain shift should prefer MC-Dropout or ensemble methods over MUS as the primary uncertainty signal, using MUS as a complementary component only when compute permits.

% ============================================================
\section{Full Metric Tables: ImageNet and MS COCO}
\label{app:full_metrics}
% ============================================================

\subsection{ImageNet-1k: Full Metric Breakdown}

Table~\ref{tab:imagenet_full} reports full secondary metrics for failure detection on ImageNet-1k across three architectures. Confidence-inv is the dominant standalone signal, achieving AUC of 0.9132, 0.9233, and 0.9131 for EfficientNet-B4, ViT-B/16, and ConvNeXt-Tiny respectively, consistent with the ablation findings. Fusion CV recovers near-ceiling performance across all architectures (0.9173, 0.9369, 0.9356), with notably lower FPR@80\%TPR (0.1237, 0.0872, 0.0700). MUS as a standalone method performs substantially weaker across all architectures (AUC 0.6415-0.7171), with Spearman correlations against Brier score declining from 0.3918 (EfficientNet-B4) to 0.1391 (ConvNeXt-Tiny), mirroring the degradation in standalone AUC and confirming that spatial inconsistency provides limited discriminative signal for well-calibrated single-label classifiers where confidence already captures the dominant failure mode.

\begin{table*}[!ht]
\caption{ImageNet-1k failure detection - full secondary metrics
(50,000 images, failure = top-1 misclassification). Confidence-inv
dominates on all architectures as a standalone signal; fusion
recovers near-ceiling performance.}
\label{tab:imagenet_full}
\centering\small
\setlength{\tabcolsep}{4pt}
\begin{tabular}{llccccc}
\toprule
\textbf{Model} & \textbf{Method} & \textbf{AUC} & \textbf{AUPR} &
\textbf{FPR@80\%} & \textbf{FPR@95\%} & \textbf{AURC} \\
\midrule
EfficientNet-B4 & MUS (Ours)     & 0.7171 & 0.4491 & 0.5062 & 0.7670 & 0.3890 \\
EfficientNet-B4 & Entropy        & 0.8961 & 0.7785 & 0.1468 & 0.4742 & 0.5252 \\
EfficientNet-B4 & Confidence-inv & 0.9132 & 0.8424 & 0.1389 & 0.4398 & 0.5442 \\
EfficientNet-B4 & Fusion CV      & 0.9173 & 0.8496 & 0.1237 & - & - \\
\midrule
ViT-B/16        & MUS (Ours)     & 0.7095 & 0.4066 & 0.4945 & 0.7555 & 0.3686 \\
ViT-B/16        & Entropy        & 0.8840 & 0.7856 & 0.1691 & 0.5967 & 0.5247 \\
ViT-B/16        & Confidence-inv & 0.9233 & 0.8781 & 0.0973 & 0.4607 & 0.5546 \\
ViT-B/16        & Fusion CV      & 0.9369 & 0.8894 & 0.0872 & - & - \\
\midrule
ConvNeXt-Tiny   & MUS (Ours)     & 0.6415 & 0.3143 & 0.5655 & 0.8059 & 0.3048 \\
ConvNeXt-Tiny   & Entropy        & 0.8194 & 0.6632 & 0.3306 & 0.6926 & 0.4787 \\
ConvNeXt-Tiny   & Confidence-inv & 0.9131 & 0.8701 & 0.1081 & 0.5654 & 0.5506 \\
ConvNeXt-Tiny   & Fusion CV      & 0.9356 & 0.8912 & 0.0700 & - & - \\
\bottomrule
\end{tabular}
\end{table*}

\subsection{MS COCO 2014: Full Metric Breakdown and Score Threshold Sensitivity}

Table~\ref{tab:coco_full} reports failure detection metrics for Faster R-CNN and RetinaNet on MS COCO 2014 ($N = 5,000$, failure defined as multi-label Brier exceeding the 75th percentile). Entropy is the dominant standalone signal for both detectors (AUC of $0.8839$ and $0.9121$), while Confidence$^{-1}$ performs substantially weaker ($0.6017$ and $0.7015$), inverting the pattern observed on ImageNet. MUS achieves moderate standalone AUC ($0.8003$ and $0.7469$), with Spearman correlations of $0.6133$ and $0.5095$ ($p < 10^{-6}$) confirming meaningful alignment with Brier-based failure  considerably stronger than in the single-label setting. Fusion CV matches or marginally exceeds entropy alone ($0.8851$ and $0.9133$), with MUS providing the primary complementary contribution, particularly for RetinaNet. Near-zero crop rates ($0.002$ and $0.001$) confirm adequate detection coverage for MUS computation across both architectures.

\begin{table*}[!ht]
\caption{MS COCO 2014 failure detection - Faster R-CNN and RetinaNet (5,000 images, failure = multi-label Brier $>$ 75th percentile). Entropy dominates as standalone signal; MUS contributes significantly in fusion, especially for RetinaNet.}
\label{tab:coco_full}
\centering\small
\setlength{\tabcolsep}{4pt}
\begin{tabular}{llccccccc}
\toprule
\textbf{Model} & \textbf{Method} & \textbf{AUC} & \textbf{ECE} &
\textbf{AUPR} & \textbf{FPR@80\%} & \textbf{FPR@95\%} &
\textbf{AURC} & \textbf{E-AURC} \\
\midrule
Faster R-CNN & MUS (Ours)     & 0.8003 & 0.0581 & 0.5835 & 0.3701 & 0.6781 & 0.10044 & 0.06919 \\
Faster R-CNN & Entropy        & 0.8839 & 0.0806 & 0.7412 & 0.2077 & 0.4512 & 0.06895 & 0.03770 \\
Faster R-CNN & Confidence-inv & 0.6017 & 0.2429 & 0.2984 & 0.6485 & 0.8512 & 0.17993 & 0.14868 \\
Faster R-CNN & Fusion CV      & 0.8851 & -    & 0.7426 & 0.2127 & -    & -     & -     \\
\midrule
RetinaNet    & MUS (Ours)     & 0.7469 & 0.0561 & 0.4501 & 0.4171 & 0.6901 & 0.11438 & 0.08313 \\
RetinaNet    & Entropy        & 0.9121 & 0.1032 & 0.8055 & 0.1568 & 0.3845 & 0.05981 & 0.02856 \\
RetinaNet    & Confidence-inv & 0.7015 & 0.1645 & 0.3777 & 0.4773 & 0.7576 & 0.13288 & 0.10163 \\
RetinaNet    & Fusion CV      & 0.9133 & -    & 0.8109 & 0.1476 & -    & -     & -     \\
\bottomrule
\end{tabular}
\end{table*}

Table~\ref{tab:coco_threshold} examines sensitivity of MUS AUC to the detection confidence threshold. Performance is fully stable between 0.05 and 0.10, with only modest degradation at 0.30 (Faster R-CNN: $-$0.0127; RetinaNet: $-$0.0446), confirming robustness to conservative thresholding choices.

\begin{table*}[!ht]
\caption{MS COCO detection confidence threshold sensitivity for MUS.
AUC is fully stable from 0.05 to 0.10 and degrades modestly at 0.30,
confirming robustness to conservative thresholding.}
\label{tab:coco_threshold}
\centering\small
\begin{tabular}{lcccc}
\toprule
\textbf{Threshold} & \textbf{FRCNN AUC} & \textbf{FRCNN AUPR} &
\textbf{RetinaNet AUC} & \textbf{RetinaNet AUPR} \\
\midrule
0.05 & 0.8003 & 0.5835 & 0.7469 & 0.4501 \\
0.10 & 0.8003 & 0.5835 & 0.7469 & 0.4501 \\
0.30 & 0.7876 & 0.5644 & 0.7023 & 0.3929 \\
\bottomrule
\end{tabular}
\end{table*}

% ============================================================
\section{Multi-Scale MUS for Nodule Detection}\label{app:multiscale}
% ============================================================

\begin{figure*}[!ht]
  \centering
  \includegraphics[width=\linewidth]{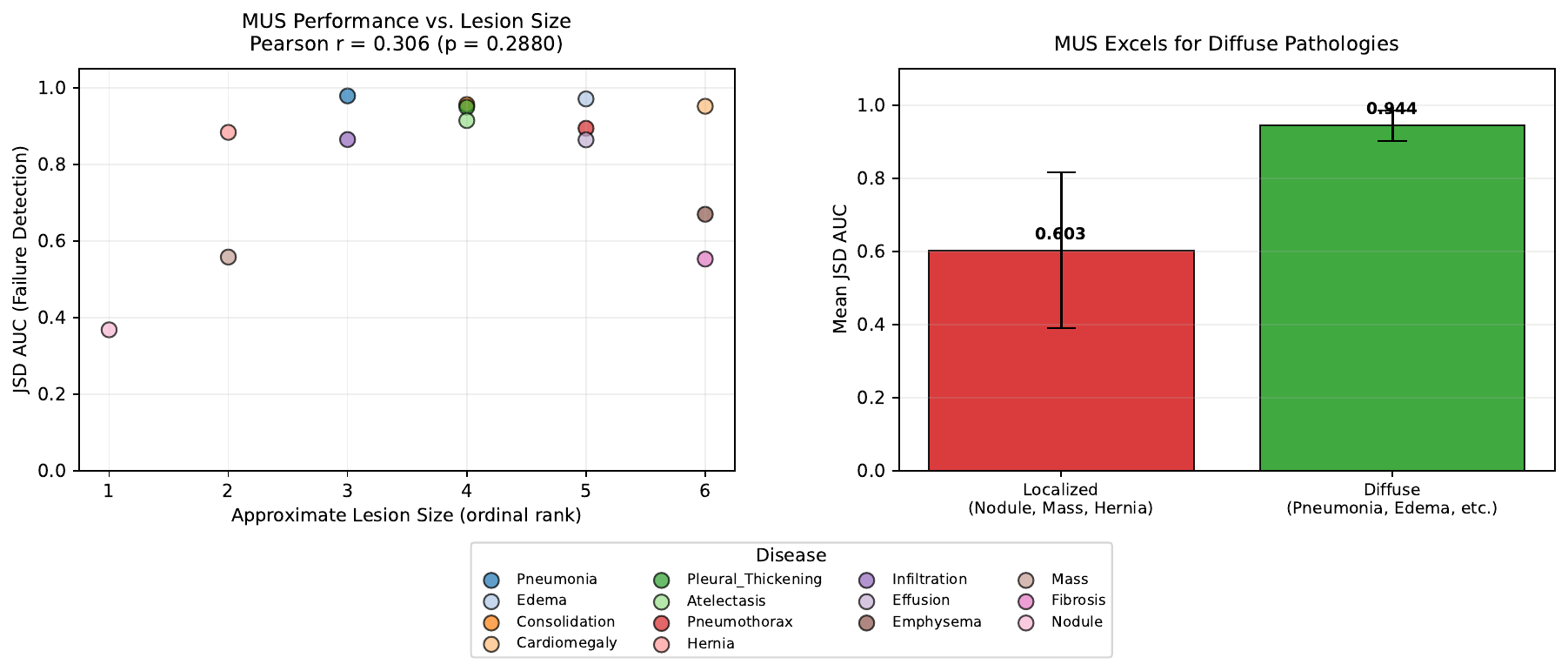}
  \caption{MUS discriminability as a function of lesion spatial extent (NIH ChestX-ray14, DenseNet-121). Per-class JSD AUC for failure detection plotted against approximate ordinal lesion size rank. Diffuse pathologies (group mean JSD AUC = 0.944) substantially outperform localized lesions (0.603), consistent with the five-crop resolution limit of ~3 cm.}
  \label{fig:pathology}
\end{figure*}

Figure~\ref{fig:pathology} shows MUS Discriminability as a Function of Lesion Spatial Extent (NIH ChestX-ray14, DenseNet-121). (Left) Per-class JSD AUC for failure detection plotted against approximate ordinal lesion size rank, where rank 1 corresponds to the most spatially localized pathologies (Nodule, Mass) and rank 6 to the most spatially diffuse (Pneumonia, Edema). Diffuse pathologies cluster at high JSD AUC (group mean of 0.944, shown in the right panel), while localized lesions cluster at low JSD AUC (group mean of 0.603). The Pearson correlation across all 14 classes is r = 0.306 (p = 0.288); the non-significant p-value reflects the bimodal rather than linear structure of the relationship  the two groups are well separated but the within-group variance is high. (Right) Group-level bar chart with 95\% bootstrap confidence intervals confirming the diffuse vs. localized split. This analysis directly motivates the multi-scale extension in Appendix \ref{app:multiscale}. Since the primary five-crop decomposition operates at 112×112 pixel resolution, pathologies smaller than approximately 3 cm in diameter at standard acquisition resolution fall below the crop resolution limit, explaining the low Nodule AUC (0.368). The multi-scale extension partially recovers performance for small focal lesions (Nodule: 0.368 → 0.589 AUC, Appendix~\ref{app:multiscale}).

\begin{figure*}[!ht]
  \centering
  \includegraphics[width=0.8\textwidth]{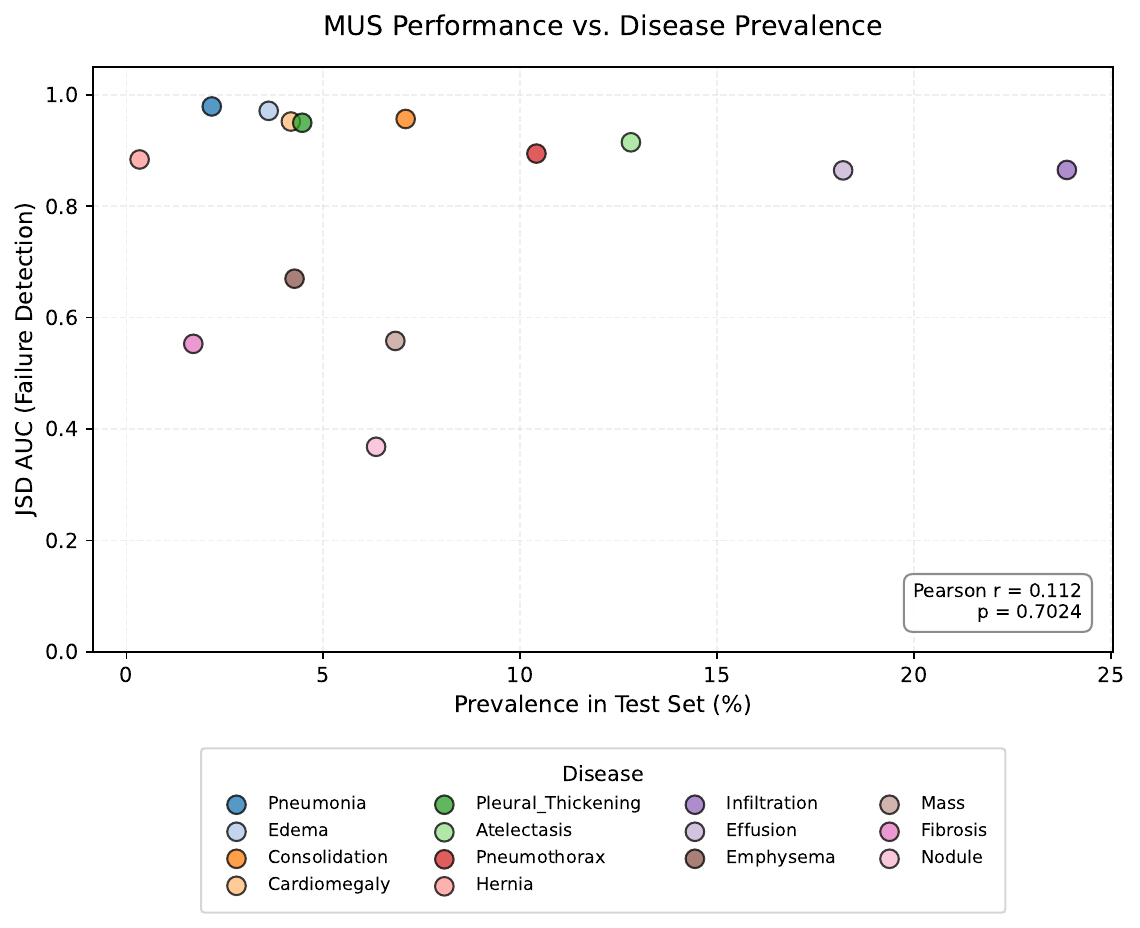}
  \caption{MUS performance vs. disease prevalence (NIH ChestX-ray14, DenseNet-121). Per-class JSD AUC for failure detection plotted against test-set prevalence (\%)}
  \label{fig:pervalence}
\end{figure*}

Figure~\ref{fig:pervalence} shows MUS Performance vs. Disease Prevalence (NIH ChestX-ray14, DenseNet-121). Per-class JSD AUC for failure detection plotted against test-set prevalence (\%), ranging from Hernia (0.34\%) to Infiltration (23.9\%). The near-zero Pearson correlation (r = 0.112, p = 0.702) confirms that MUS discriminability is not confounded by class frequency or label imbalance. This rules out the alternative explanation that commonly occurring classes benefit from better model calibration, which could artificially inflate uncertainty discrimination for frequent classes. The wide performance variation across classes of similar prevalence  for example, Pneumothorax (10.4\% prevalence, JSD AUC = 0.894) vs. Emphysema (4.3\% prevalence, JSD AUC = 0.670) confirms that lesion morphology and spatial extent, not dataset statistics, drive MUS discriminability. Rare classes such as Hernia (0.34\%) and Pneumonia (2.2\%) span both extremes of JSD AUC (0.884 and 0.979 respectively), further confirming the prevalence-independence of MUS performance.

Table~\ref{tab:multiscale} reports the multi-scale MUS extension for Nodule failure detection on NIH DenseNet-121. The baseline five-crop coarse decomposition achieves per-class JSD AUC of 0.3680 for Nodule below chance. A 16-crop fine-grained grid ($56 \times 56$ patches, covering the full image at finer resolution) partially recovers performance.

\begin{table}[!ht]
\caption{Multi-scale MUS for Nodule failure detection (NIH DenseNet-121). Fine crops use a 16-crop $56\times56$ grid. The fine-only result (0.589) is the primary reported improvement; the coarse+fine average (0.495) is lower because averaging with the weak coarse score (0.368) dilutes the signal.}

\label{tab:multiscale}
\centering\small
\setlength{\tabcolsep}{4pt}
\begin{tabular}{lc}
\toprule
\textbf{Method} & \textbf{Nodule JSD AUC} \\
\midrule
Baseline (5-crop coarse)                    & 0.3680 \\
\textbf{Fine MUS (16-crop, $56\times56$)}   & \textbf{0.5892} \\
Multi-scale avg (coarse $+$ fine)           & 0.4945 \\
\midrule
Improvement (Fine vs.\ Baseline)            & $+$0.2212 \\
\bottomrule
\end{tabular}
\end{table}

The fine-grid MUS score (0.589) substantially outperforms the coarse baseline (0.368). The multi-scale average (0.495) is lower because averaging with the weak coarse score dilutes the signal; therefore, we report the fine MUS as the primary multi-scale result throughout the paper. The remaining gap relative to diffuse pathology performance (0.85–0.98) reflects the inherent resolution limit of small focal lesion detection at standard imaging resolutions. This improvement confirms that the limitation is geometric five $112 \times 112$ crops cannot resolve nodules smaller than 3 cm rather than fundamental. The multi-scale MUS approach (using a 4×4 grid of 56×56 crops) is not adopted as the primary method due to its 16‑pass overhead; instead, it is presented as a direction for future work.

% ============================================================
\section{Clinical Utility and Demographic Fairness}\label{app:clinical}
% ============================================================

\subsection{When to Use SpatialUQ}\label{app:whentouse}

\begin{table}[!ht]
\caption{When to use SpatialUQ. This table summarizes findings already present in Sections~5--6}
\label{tab:whentouse}
\centering\small
\begin{tabular}{@{}p{3.9cm} p{3.9cm}@{}}
\toprule
\textbf{Use MUS when} & \textbf{Do not rely on MUS when} \\
\midrule
Diffuse, globally visible conditions & Highly localized lesions (nodules) \\
Overconfident multi-label classifiers & Well-calibrated single-label models \\
Moderate distribution shift & Severe covariate shift ($\rho \lesssim 0.15$) \\
No retraining or internals access & Generative VLMs with global pooling \\
\bottomrule
\end{tabular}
\end{table}

\subsection{Selective Prediction Coverage (NIH DenseNet-121)}

Table~\ref{tab:selective} reports selective prediction performance for NIH DenseNet-121 using MUS as the referral score. At a 10\% recall budget, MUS refers only 4.3\% of the test set (1,091 images) with a precision of 58.7\%, catching 640 failures demonstrating that a small, targeted referral volume captures a disproportionate share of model errors. Precision degrades modestly as recall increases (54.5\% at 20\%, 51.4\% at 30\%), reflecting the expected precision recall trade-off under a fixed failure distribution. Backbone AUC improves monotonically with referral rate, from $+0.0017$ at 5\% to $+0.0092$ at 30\% (baseline AUC $= 0.8080$), confirming that selectively abstaining on high-uncertainty predictions yields measurable downstream classification improvement. As we can see from Figure \ref{fig:selective_prediction_Pneumothorax} for Pneumothorax specifically , rejecting the top 10\% most uncertain images improves class AUC from $0.8471$ to $0.8495$ ($+0.0023$), with further gains at 20\% rejection ($0.8503, +0.0031$), illustrating that MUS-guided deferral provides clinically meaningful improvement even for already well-performing classes.

\begin{table*}[!ht]
\caption{Selective prediction - NIH DenseNet-121 (MUS as uncertainty score). Images are referred in descending MUS order. Budget = fraction of full test set referred; Errors Caught = number of detected failures out of 6,399 total.}
\label{tab:selective}
\centering\small
\begin{tabular}{ccccc}
\toprule
\textbf{Recall} & \textbf{Precision} & \textbf{Images Referred} &
\textbf{Budget} & \textbf{Errors Caught} \\
\midrule
10\% & 58.7\% & 1,091 & 4.3\%  & 640 / 6,399 \\
20\% & 54.5\% & 2,349 & 9.2\%  & 1,280 / 6,399 \\
30\% & 51.4\% & 3,732 & 14.6\% & 1,920 / 6,399 \\
\bottomrule
\end{tabular}
\end{table*}

\begin{figure*}[!ht]
  \centering
  \includegraphics[width=0.75\textwidth]{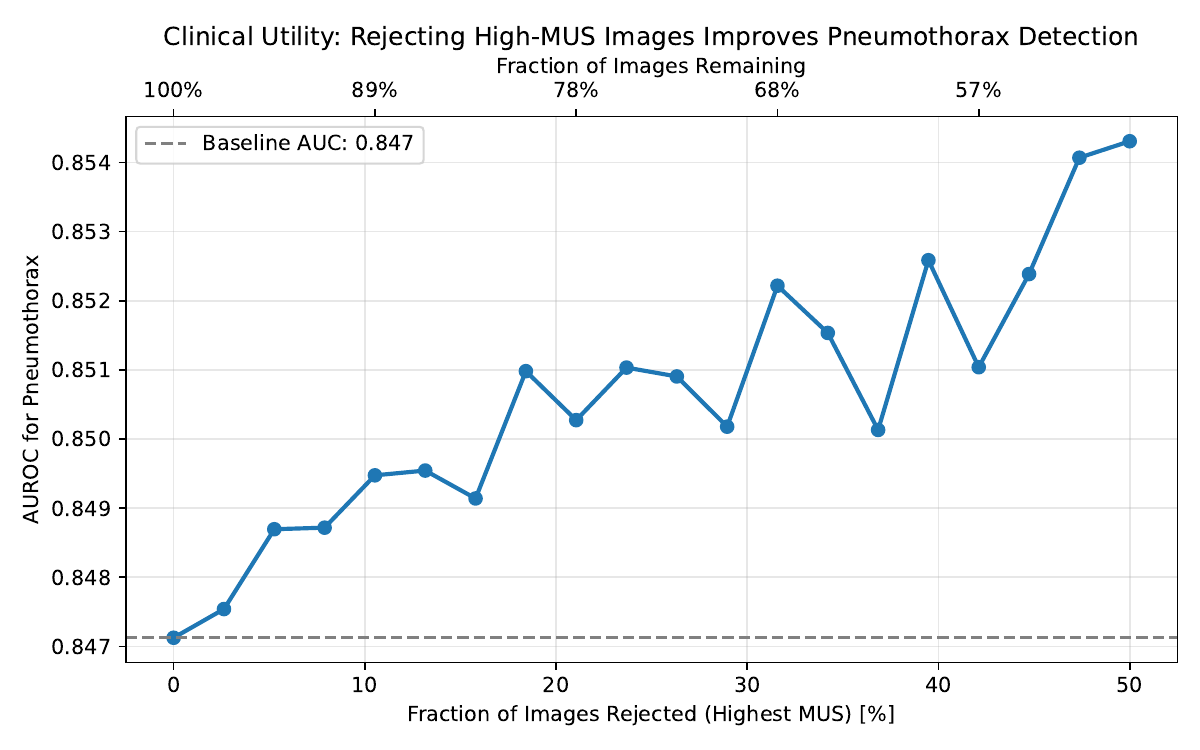}
  \caption{Selective prediction via Multi-view Uncertainty Score (MUS) for Pneumothorax detection. Performance (AUROC) shows a general upward trend as images with high spatial divergence are rejected, with local fluctuation reflecting sampling variability at intermediate rejection rates. The overall trajectory stabilizes above the baseline AUC of 0.847 at higher rejection rates, indicating that MUS effectively identifies atypical or spatially incoherent presentations whose removal improves downstream classification performance.}
  \label{fig:selective_prediction_Pneumothorax}
\end{figure*}

\subsection{Clinical Triage Simulation}

To operationalise selective prediction in a clinically realistic setting, we evaluate both a validation-calibrated 95\%-sensitivity threshold (ensuring no more than 5\% of diseased images are missed) and direct MUS-ranked referral at fixed recall levels.

\paragraph{} Under MUS-ranked referral, the NIH DenseNet-121 backbone achieves the most favourable profile: referring the 2,349 highest-MUS images (9.2\% of 25,596) recovers 20\% of Brier-defined failures at 54.5\% referral precision.

\paragraph{} EfficientNet-B4 (threshold $= 0.0801$) refers 97.7\% of images under the same 95\%-sensitivity rule (only a 2.3\% workload reduction) at a false negative rate of 0.93\%, reflecting its narrower uncertainty spread relative to DenseNet-121.

\paragraph{} ViT-B/16 (threshold $= 0.0782$) refers 96.4\% of images (a 3.6\% workload reduction [3.4\%, 3.9\%]) at a false negative rate of 1.22\%. The validation-calibrated threshold thus maintains 95\% sensitivity but barely reduces workload on EfficientNet-B4 and ViT-B/16; the MUS-ranked referral above is the practically useful operating point.

\subsection{Demographic Fairness}

Table~\ref{tab:fairness_densenet} reports MUS failure detection AUC stratified by age group and gender for DenseNet-121 on the NIH test set ($N = 25,596$). Gender parity is near-perfect (female $0.7838$ vs.\ male $0.7837$), indicating no sex-based bias in spatial uncertainty estimation. Age-stratified results reveal a modest performance gradient, with MUS AUC highest for patients under 40 ($0.8199$) and declining with age ($0.7512$ for age 60-80). We attribute this to two compounding factors: older patients exhibit higher rates of comorbid findings, producing genuinely ambiguous multi-label images where spatial inconsistency is a noisier failure signal; and the NIH training distribution skews toward middle-aged presentations, potentially yielding better-calibrated predictions for younger patients.

\begin{table}[!ht]
\caption{Demographic fairness audit - MUS failure detection AUC by age group and gender, NIH DenseNet-121 ($N = 25{,}596$). Gender parity is near-perfect (0.7838 vs.\ 0.7837). MUS AUC is slightly higher for younger patients, consistent with radiographic image quality and lesion acuity differences.}
\label{tab:fairness_densenet}
\centering\small
\begin{tabular}{lcc}
\toprule
\textbf{Subgroup} & \textbf{N} & \textbf{MUS AUC} \\
\midrule
Age $<$ 40   & 8,872  & 0.8199 \\
Age 40--60   & 11,151 & 0.7714 \\
Age 60--80   & 5,379  & 0.7512 \\
Age $>$ 80   & 190    & 0.7558 \\
\midrule
Female       & 10,714 & 0.7838 \\
Male         & 14,882 & 0.7837 \\
\midrule
\textbf{Overall} & \textbf{25,596} & \textbf{0.7837} \\
\bottomrule
\end{tabular}
\end{table}

Table~\ref{tab:fairness_effb4} reports the corresponding audit for EfficientNet-B4, which exhibits a markedly different age-stratified profile: MUS AUC increases with age ($0.7723$ for age $< 40$ vs.\ $0.7943$ for age 60-80), inverting the gradient observed for DenseNet-121. This reversal suggests that architectural differences in how compound-scaled convolutional features encode spatial inconsistency interact non-trivially with patient demographics. Gender disparity remains negligible (female $0.7813$ vs.\ male $0.7863$). Collectively, both audits confirm that MUS does not introduce systematic sex-based bias, while the architecture-dependent age gradient warrants consideration when deploying MUS across heterogeneous patient populations.

\begin{table}[!ht]
\caption{Demographic fairness audit - NIH EfficientNet-B4. Unlike DenseNet-121, EfficientNet-B4 shows slightly higher MUS AUC for older patients (Age 60--80: 0.7943 vs.\ $<$40: 0.7723), suggesting architectural differences in how spatial inconsistency manifests across patient demographics.}
\label{tab:fairness_effb4}
\centering\small
\begin{tabular}{lcc}
\toprule
\textbf{Subgroup} & \textbf{N} & \textbf{MUS AUC} \\
\midrule
Age $<$ 40   & 8,872  & 0.7723 \\
Age 40--60   & 11,151 & 0.7814 \\
Age 60--80   & 5,379  & 0.7943 \\
\midrule
Female       & 10,714 & 0.7813 \\
Male         & 14,882 & 0.7863 \\
\midrule
\textbf{Overall} & \textbf{25,596} & \textbf{0.7842} \\
\bottomrule
\end{tabular}
\end{table}

% ============================================================
\section{Implementation Details}\label{app:implementation}
% ============================================================

\paragraph{Hardware and training times.} 
All NIH experiments were run on a single NVIDIA Tesla T4 GPU (Kaggle). Training times: DenseNet-121 $\approx$8\,h; EfficientNet-B4 $\approx$10\,h; ViT-B/16 $\approx$6\,h (early stopping at epoch 17).

\paragraph{Inference cost.}
MUS requires 6 forward passes per image (1 global + 5 crops) under \texttt{torch.no\_grad()}, all runnable in a single batched call. Measured wall‑clock time on the NIH test set (25,596 images, batch size 32): DenseNet‑121 $\approx$8 min; EfficientNet‑B4 $\approx$10 min; ViT‑B/16 $\approx$29 min. MC‑Dropout (30 passes) required $\approx$39 min, $\approx$51 min, and $\approx$145 min respectively. a consistent 5$\times$ speedup matching the pass‑count reduction (6 vs.~30). A five‑member deep ensemble uses 5 forward passes (one per model), comparable to MUS at inference, but requires 5$\times$ the training cost and cannot be applied to a frozen checkpoint. MUS requires no additional training, making it suitable for frozen black‑box deployment.

\paragraph{Reproducibility.}
All experiments use fixed random seed 42, deterministic cuDNN, and result caching (\texttt{.npy}/\texttt{.csv} files). Code and data splits are included in the supplementary material and will be made publicly available upon acceptance. The ViT backbone is loaded via \texttt{timm} library (\texttt{vit\_base\_patch16\_224}). ImageNet and COCO models use standard \texttt{torchvision} pretrained checkpoints.

\paragraph{Baseline hyperparameters.}
Mahalanobis distance: PCA to 256 dimensions (retaining $\geq$95\% variance), global Gaussian fit on training features, Shrunk Covariance regularization. ODIN: temperature $T = 1000$, input perturbation $\varepsilon = 0.001$, softmax-based (not sigmoid), applied to all 14 output logits jointly. MC-Dropout: $T = 30$ stochastic forward passes with dropout active; mean predictive variance over classes used as uncertainty score. TTA: $N = 10$ augmentations (random scale 0.8-1.0, rotation $\pm10^\circ$; no horizontal flip to preserve laterality). Deep Ensemble: 5 independently trained DenseNet-121 models (seeds 42, 186, 456, 789, 1011); mean predictive entropy used.

\section{Extended Discussion and SOTA Analysis}\label{app:sota_analysis}

\subsection{Why MUS Works: Mechanism}\label{app:why_MUS_works}

MUS exploits a fundamental property of neural classifiers under uncertainty. When a model has confidently localized discriminative features, its predictions remain stable across spatial sub-regions: the full image and any crop partially containing those features will yield broadly similar output distributions, producing low JSD. When no such anchoring evidence exists  due to low image quality, atypical presentation, or distributional mismatch small changes in the spatial field of view produce large distributional shifts, yielding high JSD. This makes spatial inconsistency a more principled signal than predictive entropy: entropy reflects output-level uncertainty but cannot distinguish genuine label ambiguity from a complete absence of evidence. MUS specifically measures whether a model's confidence is spatially reproducible  a property that entropy is blind to. This mechanism holds consistently across model families. Spearman correlations between MUS and per-image Brier score on NIH are $\rho = 0.52$ (DenseNet-121), $\rho = 0.41$ (EfficientNet-B4), and $\rho = 0.44$ (ViT-B/16), all significant at $p < 10^{-6}$. The modest drop from DenseNet to EfficientNet and ViT is consistent with their more spatially coherent feature representations, which compress the dynamic range of MUS as discussed in Section \ref{sec:discussion} Score calibration further supports practical deployment: MUS achieves SCE of 0.04-0.08 across architectures, substantially better than most competing methods (Table~\ref{tab:nih_main}), confirming that MUS scores align closely with actual failure rates rather than producing well-ranked but poorly-scaled uncertainty estimates.

\begin{figure}[!ht]
  \centering
  \includegraphics[width=\linewidth]{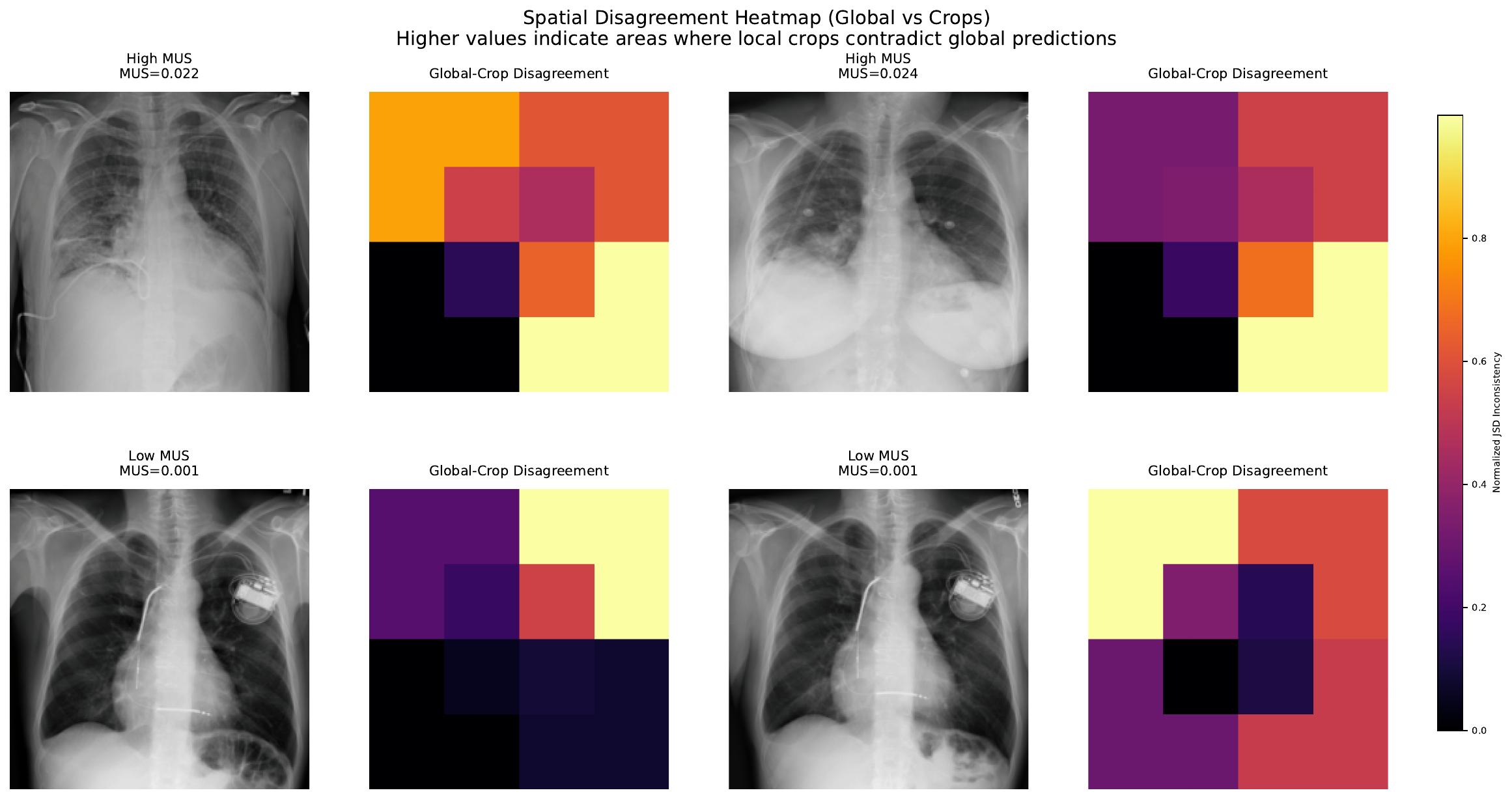}
  \caption{Spatial disagreement heatmaps. High MUS (top) shows strong per‑crop JSD    variation; low MUS (bottom) shows uniform consistency. Color scale: black (low) to yellow (high).}
  
  \label{fig:disagreement}
\end{figure}

Figure~\ref{fig:disagreement} shows Spatial Disagreement Heatmaps. Each panel shows a chest X-ray paired with its normalized per-crop JSD inconsistency map (color scale: black = low, yellow = high). On top row (High MUS), two images with MUS = 0.022 and 0.024 display strong spatial variation across quadrant crops, with particularly high disagreement localised to upper lung regions consistent with diffuse pathologies whose spatial extent changes meaningfully across crops. On bottom row (Low MUS), two images with MUS = 0.001 show uniformly low JSD across all crop positions, reflecting spatially stable model predictions regardless of field of view. The heatmaps provide interpretable spatial grounding for the scalar MUS score.

\begin{figure}[!ht]
  \centering
  \includegraphics[width=\linewidth]{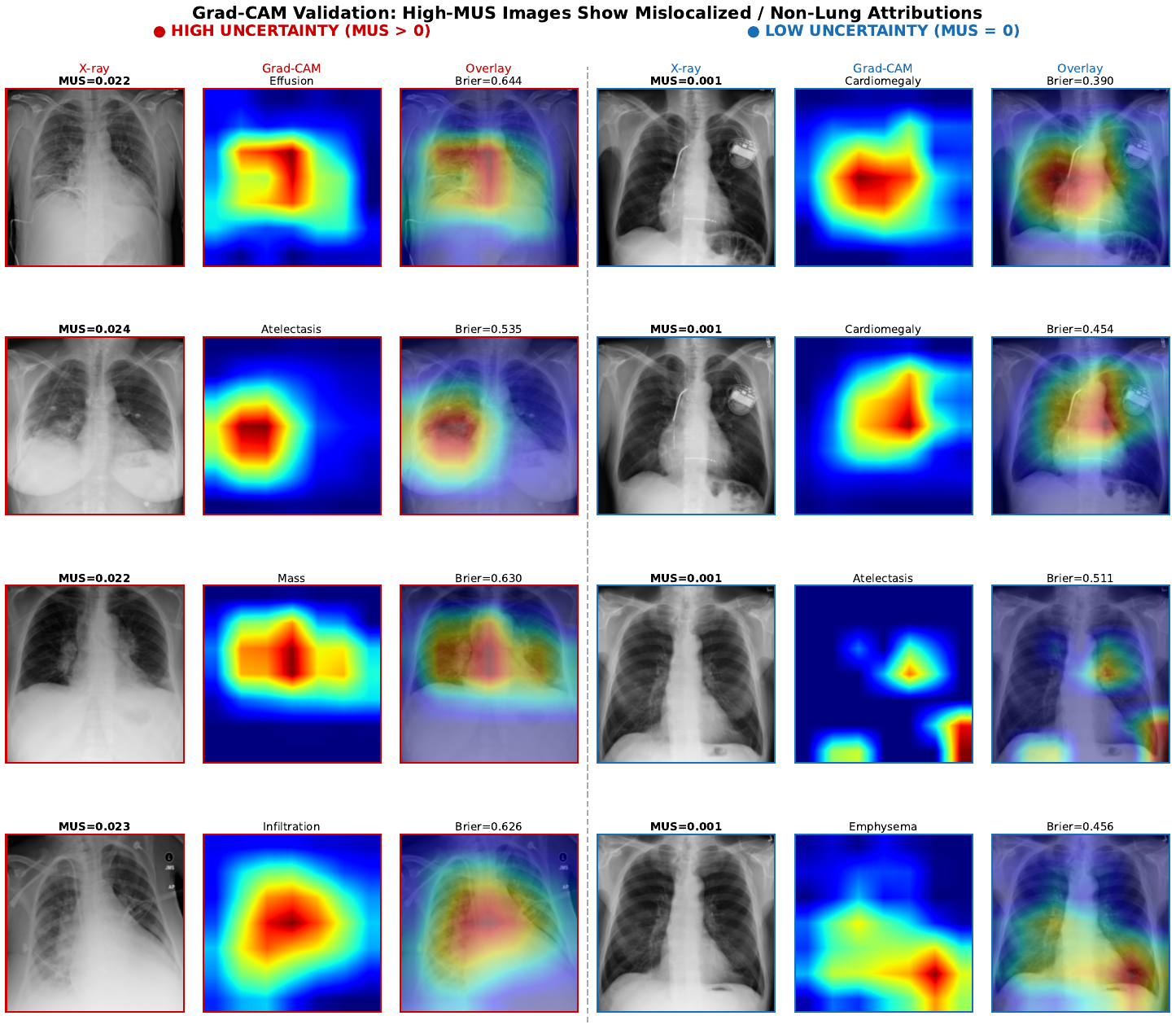}
  \caption{Grad‑CAM validation. High MUS (red border) produces diffuse, non‑lung attributions; low MUS (blue border) yields compact, localized activations.}
  \label{fig:gradcam}
\end{figure}

Figure~\ref{fig:gradcam} provides mechanistic validation for SpatialUQ via Grad-CAM attribution maps. High-MUS cases (red border; MUS = 0.022-0.024) consistently exhibit anatomically diffuse or mislocalized attributions, with activations concentrated outside the lung parenchyma, signifying a failure to anchor to discriminative features. Conversely, low-MUS cases (blue border; MUS $\approx$ 0) show compact, plausible attributions localized to expected pathology regions. This correlation between spatial inconsistency and attribution diffusion suggests that high MUS identifies instances where the model relies on spurious or non-specific features, whereas spatial consistency reflects the utilization of localized, interpretable features.

\subsection{Per-Method SOTA Analysis}

\paragraph{MC-Dropout~\citep{gal2016dropout}.}
MUS outperforms MC-Dropout in 4 of 5 settings: NIH DenseNet (+0.119 AUC), EfficientNet (+0.249), ViT (+0.100), CheXpert (+0.109), all $p < 10^{-6}$. The sole exception is VinBigData ($-0.150$ AUC), where parameter-space stochasticity captures epistemic uncertainty under extreme distribution shift more reliably than spatial consistency. MC-Dropout's primary disadvantage, 5$\times$ the inference passes of MUS (30 vs.\ 6), and requirement for training-time dropout layers are entirely avoided by MUS.

\paragraph{Deep Ensembles~\citep{lakshminarayanan2017simple}.}
The 5-member DenseNet ensemble achieves 0.813 AUC, outperforming single model MUS (0.784) but below the linear fusion (0.832, $\Delta = +0.019$, $p < 10^{-6}$). Ensembles require training 4 additional models 5$\times$ training cost and storage  making MUS + linear fusion a substantially more practical alternative in resource-constrained clinical settings.

\paragraph{ODIN~\citep{liang2017enhancing}.}
Achieves 0.759 AUC on NIH DenseNet (below MUS) but 0.800 AUC on NIH ViT the strongest single-method competitor on ViT. ODIN requires backpropagation through the input for gradient-based perturbation, adding inference overhead and gradient access requirements that MUS does not need.

\paragraph{DDU~\citep{mukhoti2021deterministic} and Mahalanobis~\citep{lee2018simple}.}
Both achieve 0.754 AUC on NIH DenseNet in our post-hoc implementation. Both require access to training-set feature distributions and storage of covariance matrices over the full training set ($\sim$73,891$\times$1,024 for DenseNet), while MUS has zero storage overhead. The original DDU additionally requires spectral normalization during training; our post-hoc approximation may underperform the full method.

\subsection{Per-Setting Performance Interpretation}

\paragraph{Architecture dependence on NIH.}
MUS achieves 0.784 (DenseNet-121), 0.784 (EfficientNet-B4), and 0.750 (ViT-B/16). The ViT gap reflects global self-attention's tendency to produce spatially coherent predictions even when incorrect, suppressing the inconsistency gradient. EfficientNet's compound scaling allocates more processing to fine-grained spatial features, explaining why its Nodule per-class JSD AUC (0.882) dramatically exceeds DenseNet's (0.368), EfficientNet's spatial inconsistency is more discriminative precisely where fine-grained resolution matters.

\paragraph{Entropy inversion on NIH.}
Raw entropy achieves AUC = 0.177 on NIH DenseNet significantly below random. The model trained with label smoothing and class-balanced BCE produces moderate entropy across most predictions; failing predictions tend to be overconfident (high probability assigned to wrong class), not uncertain, making entropy anti-correlated with failure in this configuration. This highlights that MUS and entropy are both architecture- and training-procedure-dependent, and that their combination requires care.

\paragraph{Detection vs.\ classification MUS performance.}
MUS achieves stronger standalone results on COCO (0.800 / 0.747 AUC) than on ImageNet (0.641-0.717 AUC) relative to baselines, reflecting the multi-label detection setting where spatial inconsistency across object classes is more structurally informative than in single-label recognition. Entropy dominates on COCO (0.884 / 0.912), consistent with detection models having better confidence calibration than multi-label medical classifiers.

\subsection{Computational Cost}
MUS requires 6 forward passes per image under \texttt{torch.no\_grad()}, all computable in a single batched call. Measured wall‑clock time on the NIH test set (25,596 images, batch size 32) for MUS only: DenseNet‑121 $\approx$8 min; EfficientNet‑B4 $\approx$10 min; ViT‑B/16 $\approx$29 min. For comparison, MC‑Dropout (30 passes) required $\approx$39 min, $\approx$51 min, and $\approx$145 min respectively, confirming a consistent 5$\times$ wall‑clock speedup that mirrors the pass‑count reduction. A five‑member deep ensemble uses 5 forward passes (one per model) comparable to MUS’s 6 passes at inference but requires 5$\times$ the training cost and cannot be applied to a frozen checkpoint without retraining. MUS’s inference cost is therefore competitive with ensembles while offering the decisive advantage of zero retraining.

\subsection{Future Directions}

The most important near-term extension is validation on MIMIC-CXR (227,835 radiographs, Beth Israel Deaconess), which would enable evaluation under temporal distribution shift, correlation with radiologist confidence scores in structured reports, and intersectional fairness analysis across race, insurance status, and comorbidity burden  dimensions not captured by NIH's age/gender metadata. Our finding that MUS AUC degrades for older patients (0.820 for $<$40 vs.\ 0.751 for 60--80 years) motivates a more thorough demographic audit on richer clinical metadata. Additional directions include adaptive crop placement via saliency guidance, MUS for transformer-based segmentation models (UNETR, SwinUNETR), and volumetric extension to 3D CT.

\section{Label Mapping for Zero-Shot Transfer}\label{app:label_mapping}

\subsection*{CheXpert Label Mapping (10 classes)}

The NIH ChestX-ray14 dataset contains 14 disease classes. CheXpert contains 14 classes of which 10 overlap with NIH. We map the shared classes directly by name. The four NIH classes absent from CheXpert (\textit{Emphysema}, \textit{Fibrosis}, \textit{Pleural\_Thickening}, \textit{Hernia}) are excluded from evaluation on CheXpert. Uncertain labels (U) in CheXpert are treated as negative (U-Zeros policy), consistent with prior work \citep{irvin2019chexpert}.

\subsection*{VinBigData Label Mapping (10 classes)} 

VinBigData contains 28 annotated findings. We map the 10 classes that correspond directly to NIH labels: \textit{Atelectasis}, \textit{Cardiomegaly}, \textit{Consolidation}, \textit{Effusion}, \textit{Infiltration}, \textit{Mass},
\textit{Nodule}, \textit{Pleural\_Thickening}, \textit{Pneumothorax}, \textit{Fibrosis}. The four NIH classes absent from VinBigData (\textit{Edema}, \textit{Pneumonia}, \textit{Emphysema}, \textit{Hernia}) are excluded. All VinBigData annotations are radiologist-verified bounding boxes; we convert these to image-level binary labels (positive if any box of that class is present).

\begin{remark}
The class-averaged bound inherits slack from Jensen's inequality in Step 2 when per-class JSD values are heterogeneous; the per-class bound itself is tight as $p_c, q_c \to \frac{1}{2}$. The purpose of Lemma~\ref{lemma:bound} is to establish that $s_{\mathrm{bern}}$ formally upper bounds mean absolute class-level disagreement, guaranteeing that a large MUS score \emph{necessarily} implies substantial shifts between global and local predictions, rather than to provide a numerically tight certificate. The empirical Spearman correlation between $s_{\mathrm{bern}}$ and per-image Brier score ($p < 0.0001$ across all medical benchmarks) confirms the practical utility of this relationship.
\end{remark}

\end{document}